\documentclass{article} 
\usepackage[utf8]{inputenc} 
\usepackage[T1]{fontenc}    
\usepackage[colorlinks]{hyperref}       
\usepackage{geometry} 

\usepackage{fancyhdr} 
\usepackage{amsthm}
\usepackage{url}            
\usepackage{booktabs}       
\usepackage{amsfonts}       
\usepackage{nicefrac}       
\usepackage{microtype}      
\usepackage{xcolor}         
\usepackage{multirow}
\usepackage{natbib}

\usepackage{amsmath,amsfonts,bm}

\def\eqref#1{equation~\ref{#1}}

\def\1{\bm{1}}

\DeclareMathAlphabet{\mathsfit}{\encodingdefault}{\sfdefault}{m}{sl}
\SetMathAlphabet{\mathsfit}{bold}{\encodingdefault}{\sfdefault}{bx}{n}
\newcommand{\tens}[1]{\bm{\mathsfit{#1}}}
\def\tA{{\tens{A}}}

\def\tG{{\tens{G}}}
\def\tH{{\tens{H}}}

\def\gO{{\mathcal{O}}}

\newcommand{\E}{\mathbb{E}}

\newcommand{\R}{\mathbb{R}}

\DeclareMathOperator{\Tr}{Tr}

\usepackage[colorlinks]{hyperref}
\usepackage{url}
\usepackage{natbib}
\usepackage{mathtools}
\usepackage{amsmath,amsfonts,amssymb,amsthm,commath}
\usepackage{float}
\usepackage{graphicx}
\graphicspath{ {./figures/} }
\usepackage{caption}
\usepackage{multirow}
\usepackage{booktabs}
\usepackage{subcaption}
\usepackage{tasks}
\usepackage{thmtools}
\usepackage{thm-restate}
\usepackage{tikz}
\usetikzlibrary{er,positioning}
\usetikzlibrary{fit,positioning}
\usepackage[capitalize, noabbrev]{cleveref}
\usepackage{algorithm}
\usepackage{algpseudocode}
\tikzstyle{block} = [rectangle, draw, fill=blue!20, 
    text width=12.8em, text centered, rounded corners, minimum height=4em]
\tikzstyle{line} = [draw, -latex']
\tikzstyle{cloud} = [draw, ellipse,fill=red!20, node distance=3cm,
    minimum height=2em]
    
\usepackage[inline]{enumitem}

\usepackage{multicol}
\usepackage{soul}

\theoremstyle{plain}
\newtheorem{theorem}{Theorem}[section]

\newtheorem{lemma}[theorem]{Lemma}
\newtheorem{condition}[theorem]{Condition}
\newtheorem{setup}[theorem]{Setup}

\theoremstyle{definition}
\newtheorem{definition}[theorem]{Definition}

\theoremstyle{remark}
\newtheorem{remark}[theorem]{Remark}

\def\ddefloop#1{\ifx\ddefloop#1\else\ddef{#1}\expandafter\ddefloop\fi}
\def\ddef#1{\expandafter\def\csname bb#1\endcsname{\ensuremath{\mathbb{#1}}}}
\ddefloop ABCDEFGHIJKLMNOPQRSTUVWXYZ\ddefloop
\def\ddef#1{\expandafter\def\csname c#1\endcsname{\ensuremath{\mathcal{#1}}}}
\ddefloop ABCDEFGHIJKLMNOPQRSTUVWXYZ\ddefloop

\usepackage{tikz}
\newcommand*\circled[1]{\tikz[baseline=(char.base)]{
            \node[shape=circle,draw,inner sep=2pt] (char) {#1};}}

\newenvironment{hproof}{%
  \proof}{\endproof}
\def\tr{\operatorname{tr}}

\def\matmul{\color{magenta}\texttt{MatMul}}
\def\lincomb{\color{magenta}\texttt{LinComb}}
\def\nonlin{\color{magenta}\texttt{NonLin}}
\def\trsp{\color{magenta}\texttt{Trsp}}

\def\S{\mathbb{S}}

\def\wh{\widehat}

\DeclarePairedDelimiter\autobracket{(}{)}
\newcommand{\p}[1]{\autobracket*{#1}}

\newcommand{\ip}[2]{\left\langle #1, #2 \right \rangle}

\theoremstyle{definition}

\title{On the Neural Tangent Kernel of Equilibrium Models}

\author{Zhili Feng \\
Carnegie Mellon University\\
\texttt{zhilif@andrew.cmu.edu} \\
\and
J. Zico Kolter \\
Carnegie Mellon University \\
Bosch Center for AI\\
\texttt{zkolter@cs.cmu.edu}\\
}

\begin{document}

\maketitle

\begin{abstract}
This work studies the neural tangent kernel (NTK) of the deep equilibrium (DEQ) model, a practical ``infinite-depth'' architecture which directly computes the infinite-depth limit of a weight-tied network via root-finding. Even though the NTK of a fully-connected neural network can be stochastic if its width and depth both tend to infinity simultaneously, we show that contrarily a DEQ model still enjoys a deterministic NTK despite its width and depth going to infinity at the same time under mild conditions. Moreover, this deterministic NTK can be found efficiently via root-finding. 
\end{abstract}

\section{Introduction}
Implicit models form a new class of machine learning models where instead of stacking explicit ``layers'', they output $z$ s.t $g(x,z)=0$, where $g$ can be either a fixed point equation \citep{bai2019deep}, a differential equation \citep{chen2018neural}, or an optimization problem \citep{gould2019deep}. This work focuses on deep equilibrium models, a class of models that effectively represent a ``infinite-depth'' weight-tied network with input injection. Specifically, let $f_\theta$ be a network parameterized by $\theta$, let $x$ be an input injection, DEQ finds $z^*$ such that $f(z^*, x)=z^*$, and uses $z^*$ as the input for downstream tasks. One interesting question to ask is, what will DEQs become if their widths also go to infinity? It is well-known that at certain random initialization, neural networks of various structures converge to Gaussian processes as their widths go to infinity \citep{neal1996priors,lee2017deep,yang2019tensor,matthews2018gaussian,novak2018bayesian,garriga2018deep}. Recent deep learning theory advances have also shown that in the infinite width limit, with proper initialization (the NTK initialization), training the network $f_\theta$ with gradient descent is equivalent to solving kernel regression with respect to the neural tangent kernel (NTK) \citep{arora2019exact,jacot2018neural,yang2019tensor,huang2020deep}. These kernel regimes provide important insights to understanding how neural networks work.

However, the infinite depth (denote depth as $d$) regime introduces several caveats. Since the NTK correlates with the infinite width (denote width as $n$) limit, a question naturally arises as how do we let $n, d\to\infty$? \citet{hanin2019finite} proved that as long as $d/n\in(0,\infty)$, the NTK of vanilla fully-connected neural network (FCNN) becomes stochastic. On the other hand, if we first take the $n\to\infty$, then $d\to\infty$\footnote{The computed quantity is $\lim_{d\to\infty}\lim_{n\to\infty} \Theta^{(d)}_n(x,y)$.}, \citet{jacot2019freeze} showed that the NTK of a FCNN converges either to a constant (\textit{freeze}), or to the Kronecker Delta (\textit{chaos}). In this work, we prove that with proper initialization, the NTK-of-DEQ enjoys a limit exchanging property $\lim_{d\to\infty}\lim_{n\to\infty} \Theta^{(d)}_n(x,y)=\lim_{n\to\infty}\lim_{d\to\infty} \Theta^{(d)}_n(x,y)$ with high probability, where $\Theta^{(d)}_n$ denotes the empirical NTK of a neural network with $d$ layers and $n$ neurons each layer. Intuitively, we name the left hand side ``DEQ-of-NTK'' and the right hand side ``NTK-of-DEQ''. The NTK-of-DEQ converges to meaningful deterministic fixed points that depend on the input in a non-trivial way, thus avoiding the freeze vs. chaos scenario. Furthermore, analogous to DEQ models, we can compute these kernels by solving fixed point equations, rather than iteratively applying the updates as for traditional NTK.
 We evaluate our approach and demonstrate that it matches the performance of existing regularized NTK methods.

\section{Background and Preliminaries}
A vanilla FCNN has the form $g^{(t)}=\sigma(W^{(t)}g^{(t-1)}+b^{(t)})$ for the $t$-th layer, and in principle $t$ can be as large as one wants. A weight-tied FCNN with input injection (FCNN-IJ) makes the bias term related to the original input and ties the weight in each layer by taking the form $z^{(t)}:=f(z^{(t-1)}, x)=\sigma(W z^{(t-1)}+Ux+b)$. \citet{bai2019deep} proposed the DEQ model, which can be equivalent to running an infinite-depth FCNN-IJ, but updated in a more clever way. The forward pass of DEQ is done by solving $f(z^*, x)=z^*$. For a stable system, this is equivalent to solving $\lim_{t\to\infty}f^{(t)}(z^{(0)}, x)$. The backward iteration is done by computing $df(z^*, x)/dz^*$ directly through the implicit function theorem, thus avoiding storing the Jacobian for each layer. This method traces back to some of the original work in recurrent backpropagation \citep{almeida1990learning,pineda1988generalization}, but with specific emphasis on: 1) computing the fixed point directly via root-finding rather than forward iteration; and 2) incorporating the elements from modern deep networks in the single ``layer'', such as self-attention transformers \citep{bai2019deep}, multi-scale convolutions \citep{bai2020multiscale}, etc. DEQ models achieve nearly state-of-the-art performances on many large-scale tasks including the CityScape semantic segmentation and ImageNet classification, while only requiring constant memory. 
Although a general DEQ model does not always guarantee to find a stable fixed point, with careful parameterization and update method, monotone operator DEQs can ensure the existence of a unique stable fixed point \citep{winston2020monotone}. 

The study of large width limits of neural networks dates back to \citet{neal1996priors}, who first discovered that a single-layered network with randomly initialized parameters becomes a Gaussian process (GP) in the large width limit. Such connection between neural networks and GP was later extended to multiple layers \citep{lee2017deep,matthews2018gaussian} and various other architectures \citep{yang2019tensor,novak2018bayesian,garriga2018deep}. The networks studied in this line of works are randomly initialized, and the GP kernels they induce are often referred to as the NNGP. 

A line of closely-related yet orthogonal work to ours is the mean-field theory of neural networks. This line of work studies the relation between depth and large-width networks (hence a GP kernel in limit) at initialization. 
\citet{poole2016exponential,schoenholz2016deep} showed that at initialization,
the correlations between all inputs on an infinitely wide network become either perfectly correlated (\textit{order}) or decorrelated (\textit{chaos}) as depth increases. They suggested we should initialize the neural network on the ``edge-of-chaos'' to make sure that signals can propagate deep enough in the forward direction, and the gradient does not vanish or explode during backpropagation \citep{raghu2017expressive,schoenholz2016deep}. These mean-field behaviors were later proven for various other structures like RNNs, CNNs, and NTKs as well \citep{chen2018dynamical,xiao2018dynamical,gilboa2019dynamical,hayou2019mean}. We emphasize that despite the similar appearance, our setting avoids the order vs. chaos scheme completely by adding input injection. The injection guarantees the converged NTK depends nontrivially on the inputs, as we will see later in the experiments.

While previous results hold either only at initialization or networks with only last layer trained, analogous limiting behavior was proven by \citet{jacot2018neural} to hold for fully-trained networks as well. They showed the kernel induced by a fully-trained infinite-width network is the following:
\begin{align}\label{eq:fullntk}
	\Theta(x,y) = \E_{\theta\sim \cN}\left[ \ip{\frac{\partial f(\theta ,x)}{\partial \theta} }{\frac{\partial f(\theta ,y)}{\partial \theta} }\right],
\end{align}
where $\cN$ represents the Gaussian distribution. They also gave a recursive formulation for the NTK of FCNN. \citet{arora2019exact,alemohammad2020recurrent,yang2020tensor} later provided formulation for convolutional NTK, recurrent NTK, and other structures.

One may ask what happens if both the width and the depth in a fully-trained network go to infinity. This question requires careful formulations as one should consider the order of two limits, as \citet{hanin2019finite} proved that width and depth cannot simultaneously tend to infinity and result in a deterministic NTK, suggesting one cannot always swap the two limits. An interesting example is that \citet{huang2020deep} showed that the infinite depth limit of a ResNet-NTK is deterministic, but if we let the width and depth go to infinity at the same rate, the ResNet behaves in a log-Gaussian fashion \citep{li2021future}. Meanwhile, the infinite depth limit of NTK does not always present favorable properties. It turns out that the vanilla FCNN does not have a meaningful convergence: either it gives a constant kernel or the Kronecker Delta kernel \citep{jacot2019freeze}. 

\paragraph{Our contributions.} We first show that unlike the infinite depth limit of NTK to FCNN, the \textit{DEQ-of-NTK} does not converge to a degenerate kernel. This non-trivial kernel can be computed efficiently using root-finding. Moreover, the \textit{NTK-of-DEQ} coincides with the DEQ-of-NTK under mild conditions. Although the proofs here involved infinite limits, we also show numerically that reasonably large networks converge to roughly the same quantities as predicted by theory, and we show the NTK-of-DEQ matches the performances of other NTKs on real-world datasets.


\subsection{Notation}
We write capital letter $W$ to represent matrices or tensors, which should be clear from the context, and use $[W]_i$ to represent the element of $W$ indexed by $i$. We write lower case letter $x$ to represent vectors or scalars. For $a\in\mathbb Z_+$, let $[a]=\{1,\ldots, a\}$. Denote $\sigma(x)=\sqrt{2}\max(0, x)$ as the normalized ReLU and $\dot\sigma$ its derivative (which only needs to be well-defined almost everywhere). The symbol $\sigma_a^2$ with subscript is always used to denote the variance of random variable $a$. We write $\cN(\mu, \Sigma)$ as the Gaussian distribution with mean $\mu\in\R^d$ and covariance matrix $\Sigma\in\R^{d\times d}$.  We let $\S^{d-1}$ be the unit sphere embedded in $\R^d$. We use $n, d$ to denote width and depth respectively, and write $G^{(d)}_n$ to stress $G$ has depth $d$ and width $n$, where $G$ can represent either a kernel or a neural network. We use the term \textit{empirical NTK} to represent $\ip{\frac{\partial f^{(d)}_n(\theta, x)}{\partial \theta}}{\frac{\partial f^{(d)}_n(\theta, y)}{\partial \theta}}$. We write $G^{(d)}=\lim_{n\to\infty} G^{(d)}_n$ , $G_n=\lim_{d\to\infty} G^{(d)}_n$, and $G=\lim_{n, d\to\infty} G^{(d)}_n$ to denote limits are taken. All missing proofs can be found in the appendix.

\section{NTK-of-DEQ with Fully-connected Layers}\label{sec:deqntk}
In this section, we show how to derive the NTK of the fully-connected DEQ.

Let $m$ be the input dimension, $x, y\in\S^{m-1}$ be a pair of inputs, $n$ be the width of the $h$-th layers where $h\in[d]$. Let $g^{(0)}(x)=\mathbf{0}\in\R^n$. Define the depth-$d$ approximation to a DEQ as the following:

\begin{align*}
    &  f^{(h)}_n(x) = \sqrt{\frac{\sigma_W^2}{n}} {W^{(h)}}{g^{(h-1)}(x)}+\sqrt{\frac{\sigma_U^2}{n}}{U^{(h)}}{x}+\sqrt{\frac{\sigma_b^2}{n}}{b^{(h)}},\\
    &g^{(h)}_n(x) = \sigma(f^{(h)}(x)),\ f^{(d+1)}_n(x) = \sigma_v\cdot v^Tg^{(d)}(x),
\end{align*}

where $h\in[d]$, $W^{(h)}\in\R^{n\times n}$, $U^{(h)}\in\R^{n\times m}$, $v\in\R^{n}$ are the internal weights and $b^{(h)}\in\R^n$ are the bias terms.

The actual DEQ effectively outputs 
$
	f^{(\infty)}_n=\sigma_v\cdot v^Tg^{(\infty)}_n(x):=\sigma_v\cdot v^T \p{\lim_{d\to\infty}g^{(d)}_n(x)}.
$
The forward pass is solved using root-finding or fixed point iteration, and the backward gradient is calculated using implicit function theorem instead of backpropogation.

One thing to note is that usually DEQs require tied-weights: $W^{(h)}=W$, $U^{(h)}=U$. and $b^{(h)}=b$ for all $h$. It turns out for the infinite width regime, DEQ with tied weights and DEQ without tied weights will induce the same NTK. We will discuss this point in more detail later.


Let $ \Theta_{n}^{(d)}(x,y)$ be the empirical NTK of $f_n^{(d)}$. In \cref{subsec:finitedepthntk}, we will derive for an arbitrarily fixed $d$, the ``finite depth iteration to DEQ-of-NTK'' $\Theta^{(d)}=\lim_{n\to\infty}\Theta_{n}^{(d)}$. In \cref{subec:convergencedeqntk}, we show that $\Theta^{(d)}$ converges to a deterministic DEQ-of-NTK. Furthermore, we prove that $\lim_{d\to\infty}\lim_{n\to\infty}\Theta_{n}^{(d)}=\lim_{n\to\infty}\lim_{d\to\infty}\Theta_{n}^{(d)}$ with high probability, that is, the DEQ-of-NTK equals the NTK-of-DEQ.
%
 
\subsection{Finite Depth Iteration to DEQ-of-NTK}\label{subsec:finitedepthntk}
Under the expressions in the beginning of \cref{sec:deqntk}, let us pick $\sigma_W,\sigma_U,\sigma_b\in\R$ arbitrarily in this section, and require the following NTK initialization.

\paragraph{NTK initialization.} We randomly initialize every entry of every $W, U, b, v$ from  $\cN(0, 1)$. 

The finite depth iteration to the DEQ-of-NTK can be expressed as the following:
 \begin{restatable}{theorem}{deqntk}\label{thm:deqntk}
 	Recursively define the following quantities for $h\in[d]$:
 	
	
	\begin{multicols}{2}
	\noindent
	\begin{align}
		&\Sigma^{(0)}(x, y)=x^\top y\label{eq:deqntksigmainit}\\
		& \resizebox{0.8\hsize}{!}{$\Lambda^{(h)}(x,y) = \begin{pmatrix} \Sigma^{(h-1)}(x,x) &\Sigma^{(h-1)}(x,y) \label{eq:deqntklambda}\\
			\Sigma^{(h-1)}(y,x) & \Sigma^{(h-1)}(y,y)
		\end{pmatrix}$}\\
		& \Sigma^{(h)}(x,y) = \sigma_W^2\mathop{\E}_{\substack{(u, v) \sim\\ \cN(0, \Lambda^{(h)})}}[\sigma(u)\sigma(v)]\notag\\
				&\qquad\qquad\quad+\sigma^2_Ux^\top y +\sigma^2_b	 \label{eq:deqntksigma}
		\end{align}
		\columnbreak
		\begin{align}
			& \dot\Sigma^{(h)}(x,y)=\sigma_W^2 \mathop{\E}_{\substack{(u, v)\sim\\ \cN(0, \Lambda^{(h)})}}[\dot\sigma(u)\dot\sigma(v)] \label{eq:deqntkdotsigma}\\
			& \Sigma^{(d+1)}(x,y)= \sigma_v^2\mathop{\E}_{\substack{(u, v) \sim\\ \cN(0, \Lambda^{(h)})}}[\sigma(u)\sigma(v)]\\
			& \dot\Sigma^{(d+1)}(x,y)=\sigma_v^2\mathop{\E}_{\substack{(u, v)\sim\\ \cN(0, \Lambda^{(h)})}}[\dot\sigma(u)\dot\sigma(v)]
		\end{align}
	\end{multicols}
	Then the $d$-depth iteration to the DEQ-of-NTK can be expressed as:
	\begin{align}\label{eq:deqntk}
	  \Theta^{(d)}(x,y) = \sum_{h=1}^{d+2}\left( \left(\Sigma^{(h-1)}\left(x, y\right)\right) \cdot \prod_{h^{\prime}=h}^{d+2} \dot{\Sigma}^{\left(h^{\prime}\right)}\left(x, y\right)\right),
	\end{align}
	where by convention we set $\dot\Sigma^{(d+2)}(x,y)=1$.
 \end{restatable}
 

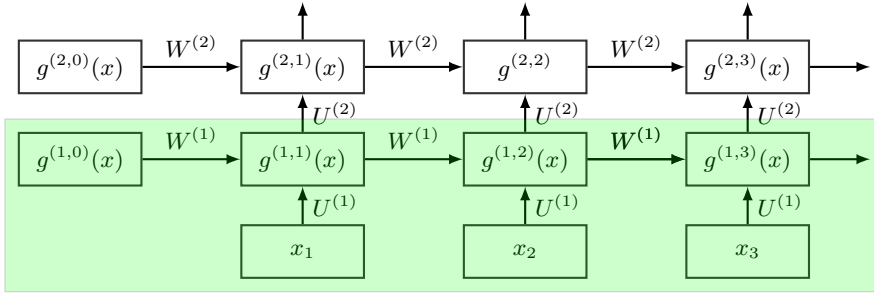
\begin{figure}[t!]
    \centering
\begin{tikzpicture}[object/.style={thin,double,<->}]
\tikzstyle{main}=[rectangle, minimum width = 16.3mm, minimum height = 7mm, thick, draw =black!80, node distance = 5mm]
\tikzstyle{main2}=[circle, minimum size = 10mm, thick, draw =black!80, node distance = 5mm]
\tikzstyle{connect}=[-latex, thick]
\tikzstyle{box}=[rectangle, draw=black!100]

  \node[main] (hlt) [] {\small$g^{(2,2)}$};
  \node[main] (hltm) [left=1.3cm of hlt] {\small $g^{(2,1)}(x)$};
  \node[main] (hltp) [right=1.3cm of hlt] {\small $g^{(2,3)}(x)$};
  \node[main] (hlmt) [below=0.5cm of hlt] {\small $g^{(1,2)}(x)$};
  \node[main] (hlmtm) [below=0.5cm of hltm] {\small $g^{(1,1)}(x)$};
  \node[main] (hlmtp) [below=0.5cm of hltp] {\small $g^{(1,3)}(x)$};
  \coordinate[right=0.8cm of hltp] (hltpp);
  \coordinate[right=0.8cm of hlmtp] (hlmtpp);
  \coordinate[above=0.5cm of hltm] (hlptm);
  \coordinate[above=0.5cm of hlt] (hlpt);
  \coordinate[above=0.5cm of hltp] (hlptp);
  
  \path (hltm) edge [connect] node[above] {\small$W^{(2)}$}(hlt)
        (hlt) edge [connect] node[above] {\small$W^{(2)}$}(hltp)
        (hlmtm) edge [connect] node[above] {\small$W^{(1)}$}(hlmt)
        (hlmt) edge [connect] node[above] {\small$W^{(1)}$}(hlmtp)
		(hlmt) edge [connect] node[above] {\small$W^{(1)}$} (hlmtp)
		(hltp) edge [connect] (hltpp)
		(hlmtp) edge [connect] (hlmtpp)
		(hltm) edge [connect] (hlptm)
		(hlt) edge [connect] (hlpt)
		(hltp) edge [connect] (hlptp)
		(hlmtm) edge [connect] node[right] {\small$U^{(2)}$}(hltm)
		(hlmt) edge [connect] node[right] {\small$U^{(2)}$}(hlt)
		(hlmtp) edge [connect] node[right] {\small$U^{(2)}$} (hltp)
		;
	\node[rectangle, fit= (hlt),opacity=.2](t) {};

	
 \node[main] (hltz) [left=1.3cm of hltm] {\small $g^{(2,0)}(x)$};
  \node[main] (hlmtz) [left=1.3cm of hlmtm] {\small $g^{(1,0)}(x)$};
  \path (hltz) edge [connect] node[above] {\small$W^{(2)}$}(hltm)
        (hlmtz) edge [connect] node[above] {\small$W^{(1)}$}(hlmtm)
        ;

\node[rectangle, fit= (hltz), opacity=.2](t) {};
\node[rectangle, fit= (hlmtz), opacity=.2](t) {};
	

\node[main] (xtm) [below=0.5cm of hlmtm] {\small $x_{1}$};
\node[main] (xt)  [below=0.5cm of hlmt] {\small $x_{2}$};
\node[main] (xtp) [below=0.5cm of hlmtp] {\small $x_{3}$};
\path (xtm) edge [connect] node[right] {\small$U^{(1)}$}(hlmtm)
(xt) edge [connect] node[right] {\small$U^{(1)}$}(hlmt)
(xtp) edge [connect] node[right] {\small$U^{(1)}$}(hlmtp);

\node[rectangle, inner sep=1.7mm,draw=black!100, fit= (hlmtz)(hlmtpp)(xtm)(xtp), fill=green,opacity=.2](t) {};

\end{tikzpicture}

    \caption{ \small
    Visualization of a simple RNN from \citet{alemohammad2020recurrent}. The green area highlights a DEQ, if $x_1, x_2,\ldots$ are all equal. 
    \vspace{-3mm}
    }
    \label{fig:rnn}
\end{figure}

One can realize that the derivation is done as if the weights in each layers are independently drawn from the previous layers, thus violating the formulation of DEQs. Nonetheless, it has been proven that under certain conditions, the tied-weight NN and untied-weight NN induce the same NTK, see \cref{rmk:tiedweights}.
\begin{remark}\label{rmk:tiedweights}
 While our derivation is done on untied weights, the NTK of its weight-tying counterpart converges to the same point. This is formally done using the Nestor program introduced in \citet{yang2019tensor,yang2020tensor}. The neural architecture needs to satisfy a gradient independent assumption. One simple check is that the output layer weights are drawn from a zero-mean Gaussian independently from any other parameters and not used anywhere in the interior of the network. This is clearly satisfied in our setting. In fact, \citet{alemohammad2020recurrent} has presented the recurrent NTK case with tied weights. Using their notation, by letting $g^{(1,0)}(\mathbf{x})=\mathbf{0}\in\R^n$, $\mathbf x$ be $T$ copies of $x$, and $T=d$ represents the depth, we exactly recover the current (finite-depth) DEQ formulation. See \cref{fig:rnn} for a visual explanation. Therefore, their conclusion directly applies to our setting. We should emphasize that our work is not a trivial extension to the recurrent NTK, because we mainly study the infinite-depth limit.
 \end{remark}

\subsection{NTK-of-DEQ equals DEQ-of-NTK}\label{subec:convergencedeqntk}
Based on \cref{eq:deqntk}, we are now ready to show what the DEQ-of-NTK $\lim_{d\to\infty}\Theta^{(d)}$ is. Then we present the main takeaway of our paper: $\lim_{d\to\infty}\Theta^{(d)} = \lim_{n\to\infty}\lim_{d\to\infty}\Theta_{n}^{(d)}$. By convention, we assume the two samples $x,y\in \S^{d-1}$, and we require the parameters $\sigma_W^2, \sigma_U^2, \sigma_b^2$ obey the following DEQ-NTK initialization:

\paragraph{DEQ-NTK initialization.} Let every entry of every $W, U, b, v$ follows the NTK initialization described in \cref{subsec:finitedepthntk}, as well as the additional requirement $\sigma_W^2+\sigma_U^2+\sigma_b^2=1$. 

Let the nonlinear activation function $\sigma$ be the normalized ReLU: $\sigma(x) = \sqrt{2}\max(0, x)$ from now on.



Using normalized ReLU along with DEQ-NTK initialization, we can derive the main convergence theorem:
\begin{restatable}{theorem}{depntkconverge}\label{thm:depntkconverge}
	Use same notations and settings in \cref{thm:deqntk}, the DEQ-of-NTK is	
	\begin{align}\label{eq:fixpointofdeqntk}
		\Theta(x,y)\triangleq\lim_{d\to\infty} \Theta^{(d)}(x, y) = \frac{ \sigma_v^2\dot\rho^*\Sigma^*(x,y)}{1-  \dot\Sigma^*(x,y)}+\sigma_v^2\rho^*,
	\end{align}
	where $\Sigma^*(x,y)\triangleq\rho^*$ is the root of $R_\sigma(\rho)-\rho$,
	\begin{align}\label{eq:dualactivationsigmawithinjection}
	\begin{split}
	  R_\sigma(\rho)
	  \triangleq\sigma_W^2\left(\frac{\sqrt{1-\rho^{2}}+\left(\pi-\cos ^{-1}\rho\right) \rho}{\pi}\right)+\sigma_U^2x^\top y+\sigma_b^2 ,
	\end{split}
	\end{align}
	and
	
	\noindent
	\begin{minipage}{0.45\textwidth}
		\begin{align}
			&\dot\rho^*\triangleq \p{\frac{\pi-\cos ^{-1}(\rho^*)}{\pi}}
		\end{align}
	\end{minipage}
	\hfill 
	\begin{minipage}{0.5\textwidth}
		\begin{align}
			&\dot\Sigma^*(x,y)\triangleq\lim_{h\to\infty }\dot\Sigma^{(h)}(x,y) =\sigma_W^2\dot\rho^*. \label{eq:dualactivationdotsigmawithinjection}
		\end{align}
	\end{minipage}
	
\end{restatable}

\begin{remark}
	Note our $\Sigma^*(x,y)$ always depends on the inputs $x$ and $y$, so the information between two inputs is always preserved, even if the depth goes to infinity. On the contrary, as pointed out by \citet{jacot2019freeze}, without input injection, $\Sigma^{(h)}(x,y)$ always converges to $1$ as $h\to\infty$, even if $x\neq y$.
\end{remark}

\Cref{thm:depntkconverge} provides us a way to direct calculate the DEQ-of-NTK by using root-finding algorithms. In practice, we can solve \cref{eq:dualactivationsigmawithinjection} by using any optimization method. Then $\Sigma^*$ and $\Theta^*$ can be computed in constant time. Since each pair of input $(x,y)$ is independent of all the other pairs, we can easily parallelize this computation process. Our derivation can be extended to more complicated structures like DEQ with convolution layers, see appendix for more detail.

One caveat of \cref{thm:depntkconverge} is the order of limits, notice that we first take the limit of the width, then the limit of the depth. Nonetheless, with sufficient conditions, one can indeed show that the limits can be exchanged, and the NTK-of-DEQ and the DEQ-of-NTK are equivalent.

\begin{restatable}{theorem}{exchangelimit}\label{thm:exchangelimit}
	Let $\sigma_W^2\leq 1/8$, $\Theta_{n}^{(d)}(x,y)=\sum_{h=1}^{d+1}\ip{\frac{\partial f(\theta ,x)}{\partial \theta^{(h)}} }{\frac{\partial f(\theta ,y)}{\partial \theta^{(h)}} }$ be the empirical NTK with depth $d$ and width $n$. Then $\lim_{n\to\infty}\lim_{d\to\infty} \Theta_{n}^{(d)}=\lim_{d\to\infty}\lim_{n\to\infty} \Theta_{n}^{(d)}$ in probability.
\end{restatable} 

\begin{hproof}
	We first use a well-established random matrix result to conclude that $\sigma_W^2<1/8$ guarantees us that $\sigma\circ\sqrt{\sigma_W^2/n}W$ is a contraction with high probability. Using this contraction property, we can then show that the empirical NTK $\Theta_n^{(d)}$ converges.  More importantly, it presents an ``uniform convergence'' property in $n$: a larger $d$ does not need a larger $n$ for the limit to converge. This is the crucial difference between this result and the results in untied-weight network. Intuitively, suppose contrarily our network has untied weights, to make our proof work we would need every layer's weight becomes a contraction. As $d$ increases, this clearly needs larger $n$ to use a union bound, which breaks if $d\to\infty$.
	
	Finally, we prove a probabilistic version of Moore-Osgood theorem to conclude that our limit exchange result holds.
\end{hproof}

\begin{remark}
	In \cref{thm:exchangelimit}, for a fixed depth $d$, $\Theta^{(d)}:=\lim_{n\to\infty}\Theta_n^{(d)}$ converges almost surely, hence we can view $\Theta:=\lim_{d\to\infty}\Theta^{(d)}$ as a constant. On the other hand, for a fixed $n$, $\Theta_n:=\lim_{d\to\infty}\Theta_n^{(d)}$ exists with probability at least $1-e^{-c\epsilon^2n}$ for some constant $c$, and $\epsilon\triangleq \frac{1-2\sqrt{2\sigma_W^2}}{\sqrt{2\sigma_W^2}}$. Formally, for any $\epsilon>0$, we have
	\[
		P\p{\abs{\Theta_n-\Theta}>\epsilon}<o(n),
	\]
	which converges in probability by definition. 
\end{remark}

\begin{remark}
	We remark that \cref{thm:exchangelimit} requires a more stringent $\sigma_W^2$ than \cref{lma:deqntkconverge}. This is indeed expected. For the actual DEQ to converge, one usually needs $I-W\succeq mI$ for some $m>0$. It seems that $\sigma_W^2\leq 1/2$ exactly reflects $I-W\succeq 0$, we leave this as an interesting future work. While \citet{hanin2019finite} also discussed about the relation between width and depth, and they concluded that the NTK may not even be deterministic if $d/n\gg 0$, our result does not contradict with theirs because their $n$ has to depend on $d$, but our proof decouples the dependency using uniform convergence thanks to weight-tying.
\end{remark}

%

\section{Case Study: Linear DEQ}\label{sec:lineardeq}

\Cref{thm:exchangelimit} shows a quite surprising result that we can safely exchange the limits, which is not at all straightforward to see. Consider the following \textsl{linear DEQ} case:
\begin{align}\label{eq:finitelineardeq}
\begin{split}
	&g_n^{(h)}(x) = \sqrt{\frac{\sigma_W^2}{n}}Wg_n^{(h-1)}(x)+\sqrt{\frac{\sigma_U^2}{n}}Ux,\ f^{(\infty)}_n(x)=v^Tg^{(\infty)}(x).
\end{split}
\end{align}

Assuming the iteration converges (this can be guaranteed with high probability picking a suitable $\sigma_W$). Equivalently, we can also write this network as 
\begin{align}\label{eq:lineardeq}
	f_n(x)=v^T\p{I-\sqrt{\frac{\sigma_W^2}{n}}W}^{-1}\sqrt{\frac{\sigma_U^2}{n}}Ux.
\end{align}

Following the same derivation in \cref{sec:deqntk}, one can easily see that $\dot\Sigma^{(h)}(x,y)=\sigma_W^2$ for all $h$, and show that 
$\lim_{d\to\infty}\lim_{n\to\infty} \Theta_n^{(d)}(x, y)=\frac{\sigma_v^2\sigma_U^2x^Ty}{(1-\sigma_W^2)^2}+\frac{\sigma_v^2\sigma_U^2x^Ty}{1-\sigma_W^2}.
$
 However, taking the infinite width limit of the network $f_n(x)$, it does not obey a Gaussian nature owing to the inverse of a shifted Gaussian matrix. It is not straightforward to see the limit exchange argument works. In this section, we aim to solve this linear DEQ case as a sanity check. In \cref{sec:simulation} we include numerical approximation that indicates the NTK-of DEQ-behaves as we expect.

\begin{restatable}{theorem}{linearexchange}\label{thm:linearexchange}
	Let $f_n(x)$ be defined as in \cref{eq:lineardeq} and $\Theta_n^{(d)}$ be the empirical NTK associated with the finite depth approximation of $f_n$ in \cref{eq:finitelineardeq}. Let $\sigma_W^2<1/4$ and $\sigma_W^2+\sigma_U^2=1$. We have
	\begin{align*}
		\lim_{d\to\infty}\lim_{n\to\infty}\Theta_n^{(d)}=\lim_{n\to\infty}\lim_{d\to\infty}\Theta_n^{(d)}
		=\frac{\sigma_v^2\sigma_U^2x^Ty}{(1-\sigma_W^2)^2}+\frac{\sigma_v^2\sigma_U^2x^Ty}{1-\sigma_W^2}
	\end{align*}
	with high probability.
\end{restatable}

\begin{hproof}
	Let $H:=\p{I-\sqrt{\frac{\sigma_W^2}{n}}W}^{-1}$. Such $H$ is well-defined with high probability if $\sigma_W^2<1/4$. A straightforward derivation gives:
	\begin{align}\label{eq:innerproductconvergence}
	\begin{split}
		&\lim_{d\to\infty}\ip{\frac{\partial f_n^{(d)}(x)}{\partial W}}{\frac{\partial f_n^{(d)}(y)}{\partial W}}
		=\frac{\sigma_U^2\sigma_v^2}{n}\frac{\sigma_W^2}{n}\ip{Hv(HUx)^T}{Hv(HUx)^T}\\
		&=\frac{\sigma_W^2\sigma_U^2}{n}\ip{HUx}{HUx}\frac{\sigma_v^2}{n}\ip{Hv}{Hv}
		\xrightarrow{p}\sigma_U^2\sigma_W^2\sigma_v^2 x^Ty \p{\frac{1}{n}\tr\p{H^TH}}^2\\
		&\xrightarrow{} \sigma_U^2\sigma_W^2\sigma_v^2 x^Ty\p{\int\frac{1}{\lambda}d\mu(\lambda)}^2,
	\end{split}
	\end{align}
	where the first convergence happens with high probability \citep{arora2019exact}, and the second convergence holds for almost every realization of a sequence of $W$. This follows from the weak convergence of probability measure $\mu_n\xrightarrow{d}\mu$ a.s. and Portmanteau lemma, where $\mu_n$ is the empirical distribution of the eigenvalue of the matrix $\p{I-\sqrt{\frac{\sigma_W^2}{n}}W}^{T}\p{I-\sqrt{\frac{\sigma_W^2}{n}}W}$. More precisely, $\mu_n=\frac{1}{n} \sum_{i=1}^{n} \delta_{\lambda_{i}}$, $\delta_{\lambda_i}$ is the delta measure at the $i$th eigenvalue $\lambda_i$.
	
	Next, we show that $\int\frac{1}{\lambda}d\mu(\lambda)=\frac{1}{1-\sigma_W^2}$. From \citet{capitaine2016spectrum}, we learn that the Stieltjes transform $g$ of $\mu$ is a root to the following cubic equation:
	\begin{align*}
		\text{For }z\in\mathbb C^+: 
		 g_\mu(z)^{-1} = \left(1- \sigma_W^{2} g_{\mu}(z)\right) z-\frac{1}{1- \sigma_W^{2} g_{\mu}(z)}.
	\end{align*}
	We then apply the inverse formula of Stieltjes transformation to derive the density
	\begin{align}\label{eq:eigendist}
		d\mu(\lambda)=\frac{1}{\pi}\lim_{b\to 0^+}\operatorname{Im}g_\mu(\lambda+ib).
	\end{align}
	This now involves a one-dimensional integration, which can be computed numerically and shown to be identical to the desired quantity.  Similarly, we can compute that 
	$$
		\lim_{d\to\infty}\ip{\frac{\partial f_n^{(d)}(x)}{\partial U}}{\frac{\partial f_n^{(d)}(y)}{\partial U}}\xrightarrow{p}  \frac{\sigma_v^2\sigma_U^2x^Ty}{1-\sigma_W^2}, \ 
		\lim_{d\to\infty}\ip{\frac{\partial f_n^{(d)}(x)}{\partial v}}{\frac{\partial f_n^{(d)}(y)}{\partial v}}\xrightarrow{p}  \frac{\sigma_v^2\sigma_U^2x^Ty}{1-\sigma_W^2}.
	$$
	Summing the three relevant terms and use the fact that $\sigma_U^2+\sigma_W^2=1$, we get the claimed result.
\end{hproof}


%
%

\section{Simulations}\label{sec:simulation}
\begin{figure*}[h]
	\centering
	\includegraphics[scale=0.35]{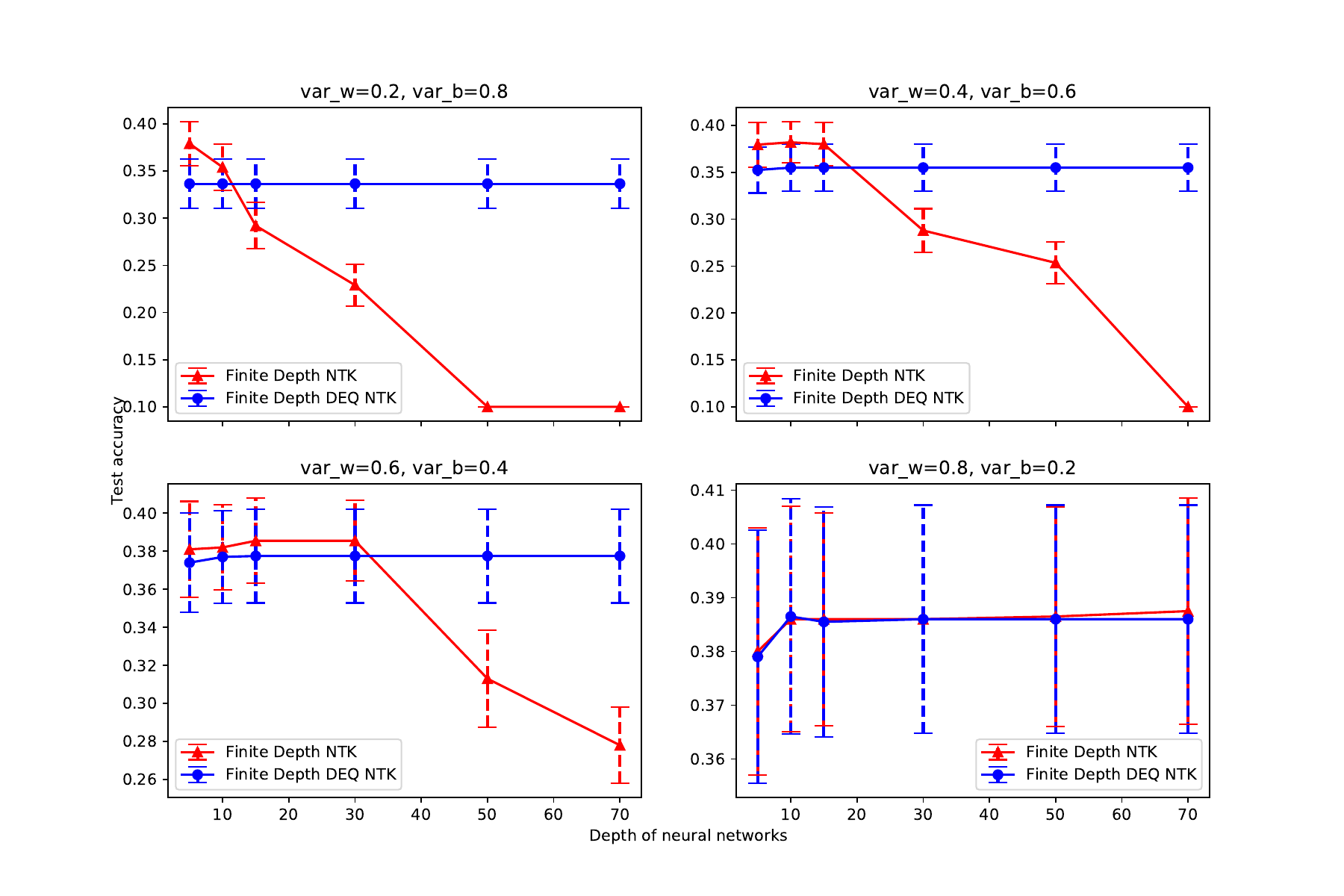}
	\vskip 0.1in
	\caption{Finite depth NTK vs. finite depth iteration of NTK-of-DEQ. In all experiments, the NTK is initialized with $\sigma_W^2$ and $\sigma_b^2$ in the title. For NTK-of-DEQ we set $\sigma_U^2=\sigma_b^2-0.1$ in the title, and $\sigma_b^2=0.1$. All models are trained on 1000 CIFAR-10 data and tested on 100 test data for $20$ random draws. The error bar represents the $95\%$ confidence interval (CI). As expected, as the depth increases, the performance of NTKs drop, eventually their $95\%$ CI becomes a singleton, yet the performance of DEQs stabilize. Also note with larger $\sigma_W^2$, the freezing of NTK takes more depths to happen.}\label{fig:finitedepth}
	\vskip 0.1in
\end{figure*}

In this section, we perform numerical simulations on both synthetic data and real-world datasets including MNIST and CIFAR-10 to demonstrate our arguments. In particular, we show that 
\begin{enumerate*}[label=(\alph*), series = tobecont, itemjoin = \ ]
	\item The NTK-of-DEQ and DEQ-of-NTK coincides, for both linear and non-linear cases,
	\item A vanilla NTK of FCNN is degenerate while the NTK-of-DEQ escapes the freeze vs. chaos scheme,
	\item The NTK-of-DEQ delivers reasonable performances on real-world datasets as a further evidence to its nondegeneracy.
\end{enumerate*}

\subsection{NTK-of-DEQ vs DEQ-of-NTK}
Recall in \cref{sec:lineardeq}, the distribution $\mu$ in \cref{eq:eigendist} is that of the eigenvalues of $H^{-T}H^{-1}\triangleq(I-\sqrt{\sigma_W^2/n}W)^{T}(I-\sqrt{\sigma_W^2/n}W)$ as $n\to\infty$. The exact limiting eigenvalue distribution $\mu$ when $\sigma_W^2=0.25, 0.5, 0.75$ is shown in \cref{fig:dist_randmat}. Keep in mind that $d\mu$ depicts the probability density of how large an eigenvalue of our random matrix can be. 

 For $\sigma_W^2=0.25, 0.5, 0,75$ we include an empirical eigenvalue distribution of $H^{-T}H^{-1}\in\R^{n\times n}$ for $n=1000$ in \cref{fig:empirical_dist_randmat}. One can see that the empirical density is sufficiently close to the limiting distribution for large enough $n$, verifying the computation in \cref{eq:eigendist}. 
 
 We calculated the empirical trace of $\frac{1}{n}\tr{H^TH}$ where $H$ is of size $5000\times 5000$. This expression is the key element for \cref{eq:innerproductconvergence}.  The simulation samples $H$ i.i.d $10$ times and the results are presented in \cref{fig:emp_vs_theory_trace}. We can see that the variance of the estimator $1/(1-\sigma_W^2)$ is negligible for small $\sigma_W^2$. Note that in the proof  we require that $\norm{\sqrt{\sigma_W^2/n}W}<1$ with high probability, which holds when $\sigma_W^2<1/4$. However, empirically the convergence of empirical trace holds for much larger $\sigma_W^2$ as well. 

We also test the difference between the empirical NTK-of-DEQ $\Theta_{n}$ and the DEQ-of-NTK $\Theta$ numerically, for both linear DEQ and nonlinear DEQ with normalized ReLU. We initialize both networks at variable width, with $\sigma_v^2=2$, $\sigma_W^2=1/8$, and $\sigma_U^2=7/8$. $\Theta_n$ is calculated by taking the inner product between the exact gradients\footnote{The gradient is taken via implicit function theorem, see details in \citet{bai2019deep}.} of a finite-width DEQ on two inputs, and $\Theta$ is computed using the DEQ-of-NTK formula in \cref{thm:depntkconverge}. A pair of input $(x,y)$ is randomly sampled and fixed throughout the simulation. For each width $n$, $10$ trials are run, and we draw the mean of $\log\frac{|\Theta-\Theta_n|}{\Theta}$ in \cref{fig:deq_approx}. The convergence of the relative residue indicates that the NTK-of-DEQ and the DEQ-of-NTK coincide as proven.

\begin{figure}

 \captionsetup[subfigure]{width=0.9\linewidth}
	\begin{subfigure}[t]{0.33\textwidth}
		\centering
		\includegraphics[width=\textwidth]{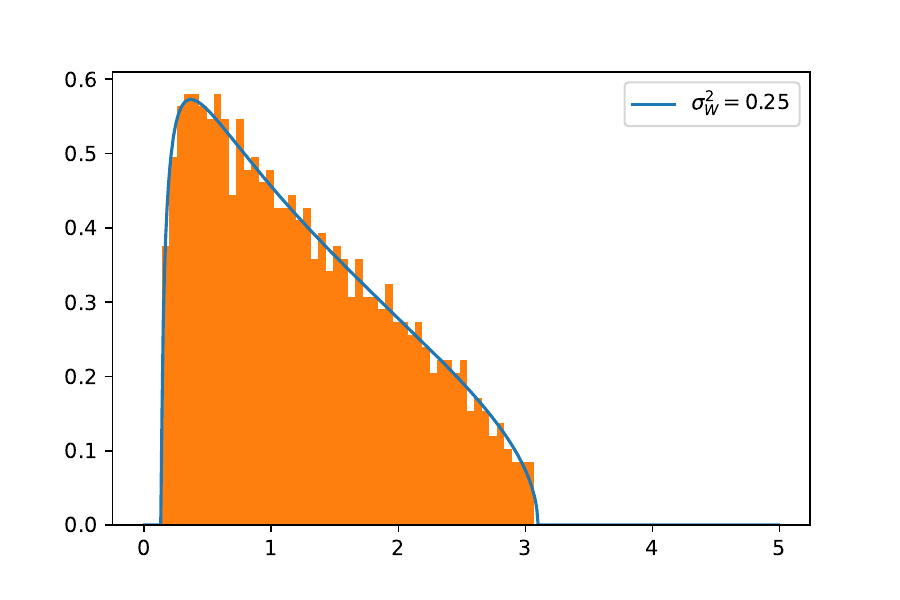}
	\end{subfigure}
	\begin{subfigure}[t]{0.33\textwidth}
		\centering
		\includegraphics[width=\textwidth]{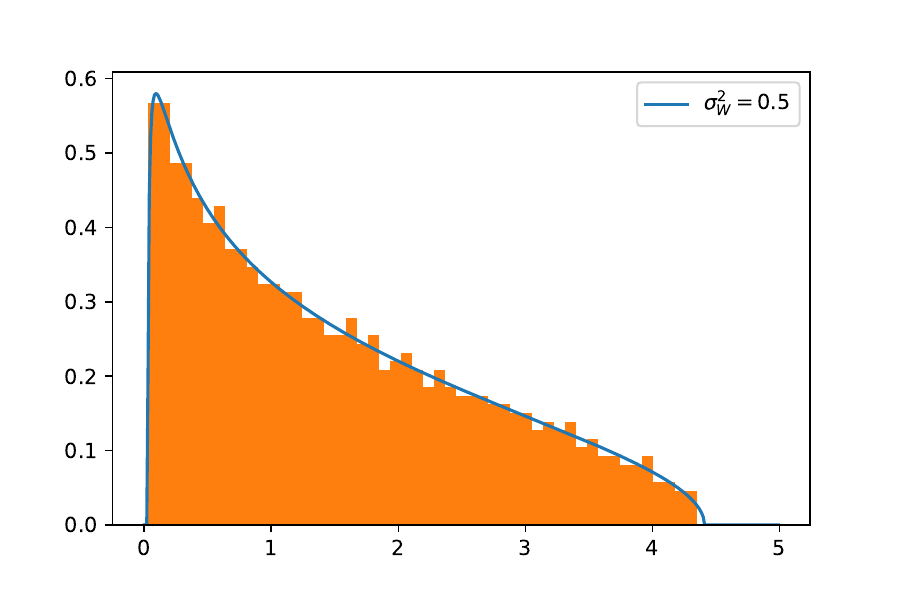}
	\end{subfigure}
	\begin{subfigure}[t]{0.33\textwidth}
		\centering
		\includegraphics[width=\textwidth]{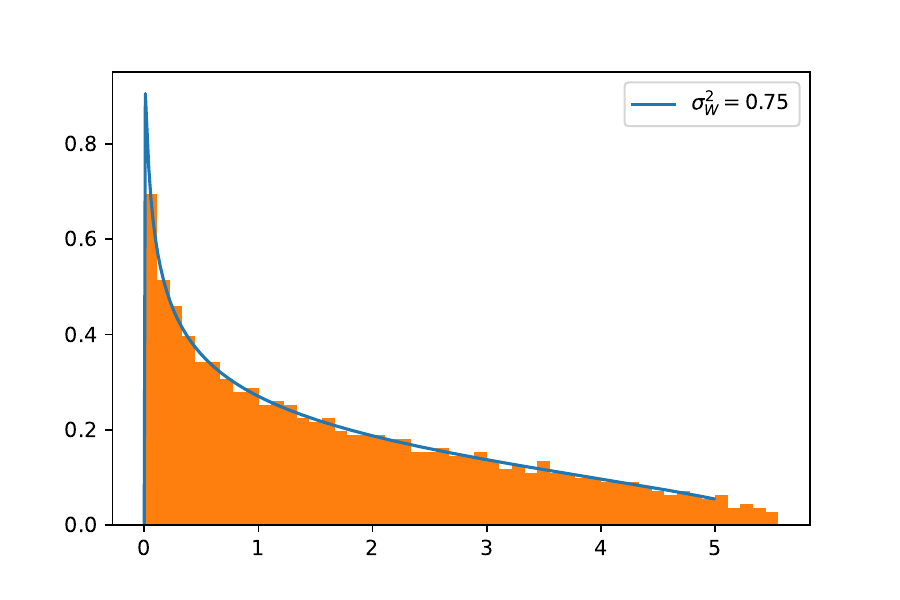}
	\end{subfigure}
	\caption{The empirical eigenvalue distribution of an instance of a $1000\times 1000$ random matrix $(I-\sqrt{\sigma_W^2/n}W)^{T}(I-\sqrt{\sigma_W^2/n}W)$ with $\sigma_W^2=0.25, 0.5, 0.75$, respectively.}
	 \label{fig:empirical_dist_randmat}

\end{figure}

\begin{figure}
 \captionsetup[subfigure]{width=0.9\linewidth}
	\centering
	\begin{subfigure}[t]{0.35\textwidth}
		\centering
		\includegraphics[width=\textwidth]{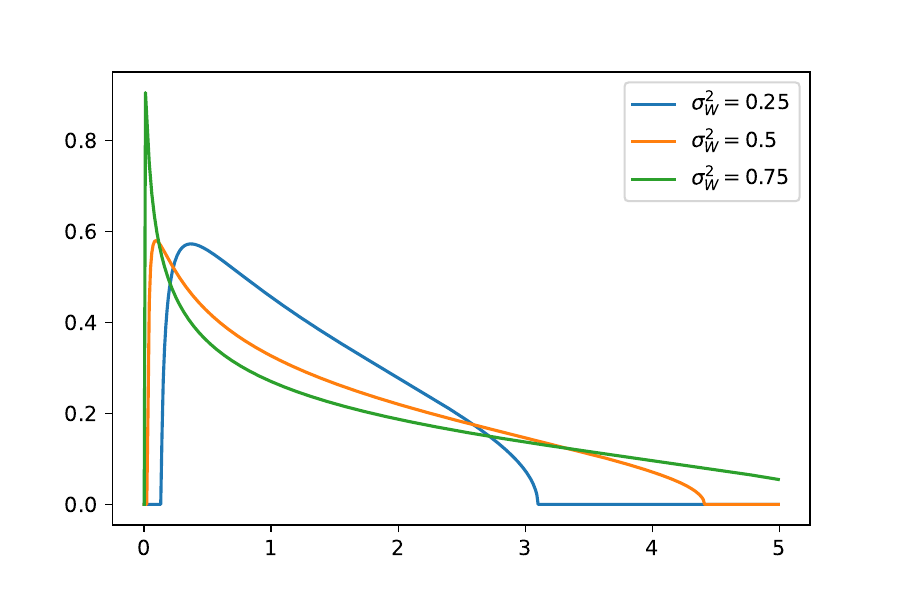}
		\caption{The limiting eigenvalue distribution of $(I-\sqrt{\sigma_W^2/n}W)^{T}(I-\sqrt{\sigma_W^2/n}W)$.}
		\label{fig:dist_randmat}
	\end{subfigure}
	\begin{subfigure}[t]{0.35\textwidth}
		\centering
		\includegraphics[width=\textwidth]{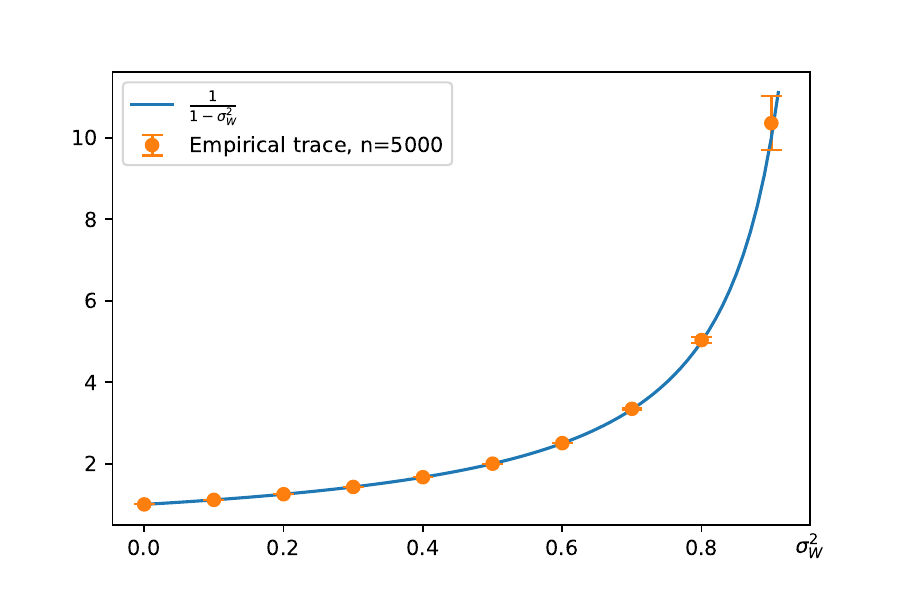}
		\caption{The empirical and expected trace. The simulation is run $10$ times the error bar denotes the standard deviation.}
		\label{fig:emp_vs_theory_trace}
	\end{subfigure}
	\caption{Demonstrations of the limiting eigenvalue distribution of $H^{-T}H^{-1}$ and its approximation.}
\end{figure}

\begin{figure}[h]
	\centering
	\begin{subfigure}[t]{0.35\linewidth}
		\centering
		\includegraphics[width=\textwidth]{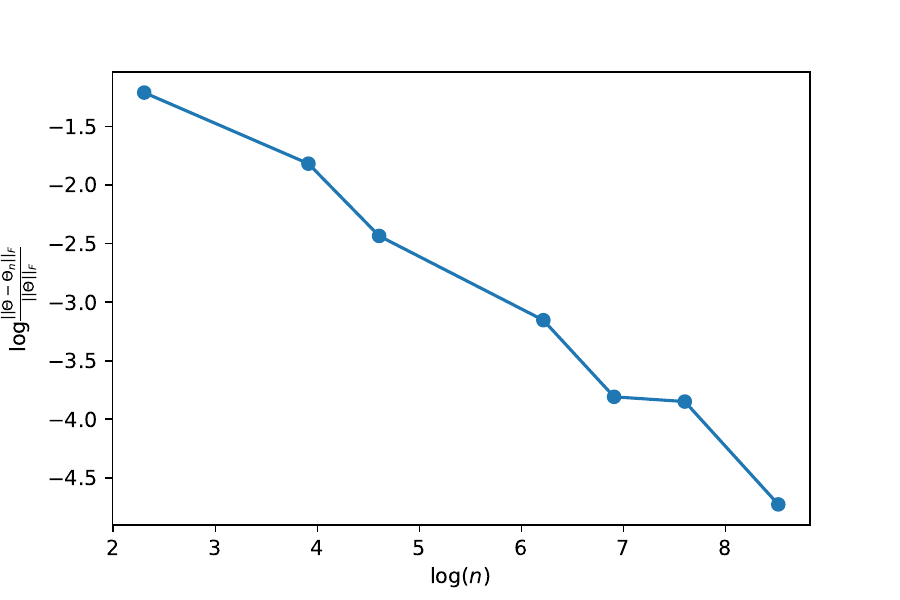}
	\end{subfigure}
	\begin{subfigure}[t]{0.35\linewidth}
		\centering
		\includegraphics[width=\textwidth]{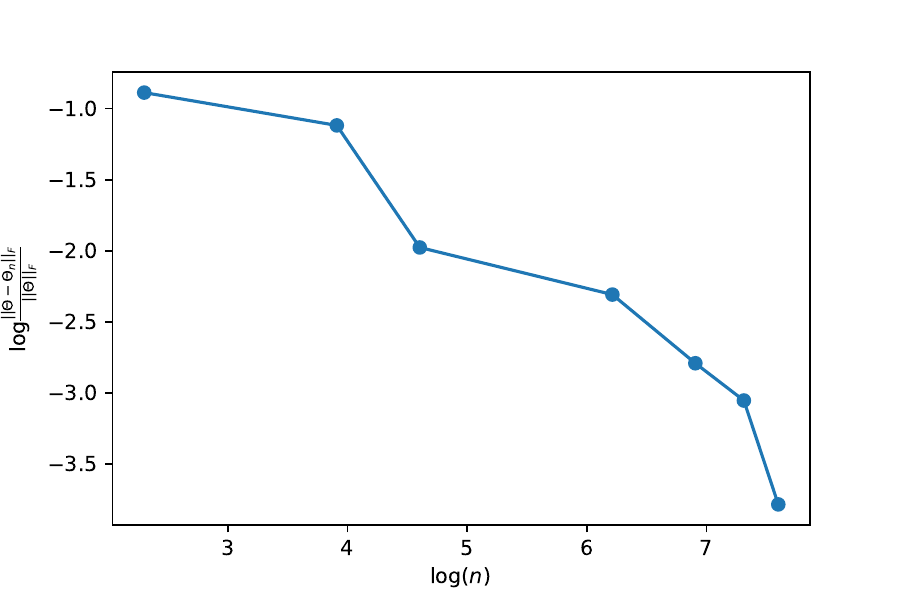}
	\end{subfigure}
 	\caption{The deviation between the empirical NTK-of-DEQ and the exact DEQ-of-NTK on a log scale. The result of linear DEQ is on the left and the result of nonlinear DEQ is on the right. We randomly sample one pair of $(x,y)$ on the unit sphere, and for each width $n$, $10$ trials are done with freshly sampled network weights, then we record the mean of relative residues in each setting. The convergence shows that NTK-of-DEQ and DEQ-of-NTK coincide.}
	\label{fig:deq_approx}
\end{figure}

\subsection{Simulations on CIFAR-10 and MNIST}

\textbf{Hyperparameter sensitivity.} We have three tunable parameters: $\sigma_W^2, \sigma_U^2, \sigma_b^2$. We try three random combinations listed in \cref{table:cifar}. As the results suggest, the performances of NTK-of-DEQ are insensitive to these parameters. This observation aligns with the description in \citet{lee2020finite}.

\begin{figure}[!t]
\begin{minipage}{0.46\textwidth}
\small
\makeatletter
\renewcommand{\@captype}{table}
\caption{Performance of NTK-of-DEQ on MNIST and CIFAR-10 dataset.}\label{table:cifar} 

	\centering
	\begin{tabular}{ ccc } 
	\midrule
	\textbf{Parameters} & \textbf{Dataset} & \textbf{Acc.} \\
	\midrule
	 $\sigma_W^2=\sigma_U^2=0.25, \sigma_b^2=0.5$ &CIFAR-10& $59.08\%$ \\ \addlinespace[1pt]
	$\sigma_W^2=0.6, \sigma_U^2=0.4, \sigma_b^2=0$ &CIFAR-10& 59.77\% \\ \addlinespace[1pt]
	$\sigma_W^2=0.8, \sigma_U^2=0.2, \sigma_b^2=0$ &CIFAR-10& $59.43\%$ \\ \addlinespace[1pt]
	$\sigma_W^2=0.6, \sigma_U^2=0.4$ &MNIST & 98.6\% \\ \addlinespace[1pt]
	
	\hline
	\end{tabular}
\makeatother
\end{minipage}
\hspace{0.07\textwidth}
\begin{minipage}{0.45\textwidth}
	\centering
	\includegraphics[width=0.8\linewidth]{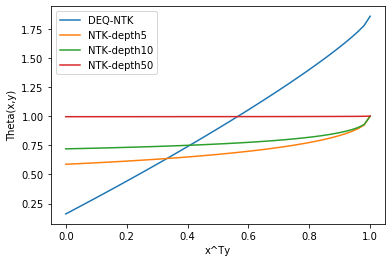}
	\caption{Relation between $\Theta(x,y)$ and $x^Ty$. }\label{fig:thetavsdotprod}
\end{minipage}
\end{figure}

\textbf{Training details and results.}
For NTK-of-DEQ, following the theory, we normalize the dataset such that each data point has unit length.
The fixed point $\Sigma^*(x,y)$ is solved by using the modified Powell hybrid method \citep{powell1970hybrid}. Notice these root finding problems are one-dimensional, hence can be quickly solved. 


After obtaining the NTK matrix, we apply kernel regressions (without regularization unless stated otherwise). For any label $y\in\{1,\ldots, n\}$, denote its one-hot encoding by $\mathbf{e}_y$. Let $\mathbf{1}\in\R^n$ be an all-$1$ vector, we train on the new encoding $-0.1\cdot \mathbf{1}+\mathbf{e}_y$. That is, we change the ``$1$'' to $0.9$, and the ``$0$'' to $-0.1$, as suggested by \cite{novak2018bayesian}. The results are listed in \cref{table:cifar}. These results prove that the NTK-of-DEQ is indeed non-degenerate.

%




On a smaller dataset with 1000 training data and 100 test data from CIFAR-10, we evaluate the performance of NTK and the finite depth iteration of NTK-of-DEQ, as depth increases. See \cref{fig:finitedepth}. When the depth increases, the performance of finite depth NTK gradually drops, eventually to 0.1 with 0 standard deviation. Also with larger $\sigma_W^2$, the degeneration of NTK occurs slower. This shows that large $\sigma_W^2$ preserves information from previous layers. \cref{fig:thetavsdotprod} also shows that the vanilla NTK becomes independent of the input inner product $x^Ty$ as the depth increases. As proven in \citet{jacot2019freeze}, the NTK will always ``freeze'' using the sets of parameters in \cref{fig:finitedepth}. In this scenario, the NTK Gram matrix becomes linearly independent as the depth increases, and its kernel regression does not have a unique solution. To circumvent this unsolvability, we add a regularization term $r\propto \frac{\epsilon \Theta (x,x)}{n}$, where $n$ is the size of the training data.

\section{Conclusion}\label{sec:conclusion}
We derive NTKs for DEQ models, and show that they can be computed efficiently via root-finding based on a limit exchanging argument. This argument is proven theoretically for non-linear DEQs and an extra sanity check is done on linear DEQs, exploiting random matrix theory. Numerical simulations are performed to demonstrate that the limit exchanging phenomenon holds for both linear and non-linear NTK-of-DEQs. Our analysis also shows that one can avoid the freeze and chaos phenomenon in infinitely deep NTKs by using input injection. Additions experiments are conducted to show that NTK-of-DEQs are non-degenerate on real-world datasets, while finite depth NTKs gradually degenerate as their depth increases. 

\bibliography{main_deqntk.bib}

\begin{thebibliography}{34}
\providecommand{\natexlab}[1]{#1}
\providecommand{\url}[1]{\texttt{#1}}
\expandafter\ifx\csname urlstyle\endcsname\relax
  \providecommand{\doi}[1]{doi: #1}\else
  \providecommand{\doi}{doi: \begingroup \urlstyle{rm}\Url}\fi

\bibitem[Alemohammad et~al.(2020)Alemohammad, Wang, Balestriero, and
  Baraniuk]{alemohammad2020recurrent}
Sina Alemohammad, Zichao Wang, Randall Balestriero, and Richard Baraniuk.
\newblock The recurrent neural tangent kernel.
\newblock \emph{arXiv preprint arXiv:2006.10246}, 2020.

\bibitem[Almeida(1990)]{almeida1990learning}
Luis~B Almeida.
\newblock A learning rule for asynchronous perceptrons with feedback in a
  combinatorial environment.
\newblock In \emph{Artificial neural networks: concept learning}, pp.\
  102--111. 1990.

\bibitem[Arora et~al.(2019)Arora, Du, Hu, Li, Salakhutdinov, and
  Wang]{arora2019exact}
Sanjeev Arora, Simon~S Du, Wei Hu, Zhiyuan Li, Russ~R Salakhutdinov, and
  Ruosong Wang.
\newblock On exact computation with an infinitely wide neural net.
\newblock In \emph{Advances in Neural Information Processing Systems}, pp.\
  8141--8150, 2019.

\bibitem[Bai et~al.(2019)Bai, Kolter, and Koltun]{bai2019deep}
Shaojie Bai, J~Zico Kolter, and Vladlen Koltun.
\newblock Deep equilibrium models.
\newblock In \emph{Advances in Neural Information Processing Systems}, pp.\
  690--701, 2019.

\bibitem[Bai et~al.(2020)Bai, Koltun, and Kolter]{bai2020multiscale}
Shaojie Bai, Vladlen Koltun, and J~Zico Kolter.
\newblock Multiscale deep equilibrium models.
\newblock \emph{arXiv preprint arXiv:2006.08656}, 2020.

\bibitem[Boucheron et~al.(2013)Boucheron, Lugosi, and
  Massart]{boucheron2013concentration}
St{\'e}phane Boucheron, G{\'a}bor Lugosi, and Pascal Massart.
\newblock \emph{Concentration inequalities: A nonasymptotic theory of
  independence}.
\newblock Oxford university press, 2013.

\bibitem[Capitaine \& Donati-Martin(2016)Capitaine and
  Donati-Martin]{capitaine2016spectrum}
Mireille Capitaine and Catherine Donati-Martin.
\newblock Spectrum of deformed random matrices and free probability.
\newblock \emph{arXiv preprint arXiv:1607.05560}, 2016.

\bibitem[Chen et~al.(2018{\natexlab{a}})Chen, Pennington, and
  Schoenholz]{chen2018dynamical}
Minmin Chen, Jeffrey Pennington, and Samuel~S Schoenholz.
\newblock Dynamical isometry and a mean field theory of rnns: Gating enables
  signal propagation in recurrent neural networks.
\newblock \emph{arXiv preprint arXiv:1806.05394}, 2018{\natexlab{a}}.

\bibitem[Chen et~al.(2018{\natexlab{b}})Chen, Rubanova, Bettencourt, and
  Duvenaud]{chen2018neural}
Ricky~TQ Chen, Yulia Rubanova, Jesse Bettencourt, and David~K Duvenaud.
\newblock Neural ordinary differential equations.
\newblock In \emph{Advances in neural information processing systems}, pp.\
  6571--6583, 2018{\natexlab{b}}.

\bibitem[Cho \& Saul(2009)Cho and Saul]{cho2009kernel}
Youngmin Cho and Lawrence Saul.
\newblock Kernel methods for deep learning.
\newblock \emph{Advances in neural information processing systems}, 22, 2009.

\bibitem[Garriga-Alonso et~al.(2018)Garriga-Alonso, Rasmussen, and
  Aitchison]{garriga2018deep}
Adri{\`a} Garriga-Alonso, Carl~Edward Rasmussen, and Laurence Aitchison.
\newblock Deep convolutional networks as shallow gaussian processes.
\newblock \emph{arXiv preprint arXiv:1808.05587}, 2018.

\bibitem[Gilboa et~al.(2019)Gilboa, Chang, Chen, Yang, Schoenholz, Chi, and
  Pennington]{gilboa2019dynamical}
Dar Gilboa, Bo~Chang, Minmin Chen, Greg Yang, Samuel~S Schoenholz, Ed~H Chi,
  and Jeffrey Pennington.
\newblock Dynamical isometry and a mean field theory of lstms and grus.
\newblock \emph{arXiv preprint arXiv:1901.08987}, 2019.

\bibitem[Gould et~al.(2019)Gould, Hartley, and Campbell]{gould2019deep}
Stephen Gould, Richard Hartley, and Dylan Campbell.
\newblock Deep declarative networks: A new hope.
\newblock \emph{arXiv preprint arXiv:1909.04866}, 2019.

\bibitem[Hanin \& Nica(2019)Hanin and Nica]{hanin2019finite}
Boris Hanin and Mihai Nica.
\newblock Finite depth and width corrections to the neural tangent kernel.
\newblock \emph{arXiv preprint arXiv:1909.05989}, 2019.

\bibitem[Hayou et~al.(2019)Hayou, Doucet, and Rousseau]{hayou2019mean}
Soufiane Hayou, Arnaud Doucet, and Judith Rousseau.
\newblock Mean-field behaviour of neural tangent kernel for deep neural
  networks.
\newblock \emph{arXiv preprint arXiv:1905.13654}, 2019.

\bibitem[Huang et~al.(2020)Huang, Wang, Tao, and Zhao]{huang2020deep}
Kaixuan Huang, Yuqing Wang, Molei Tao, and Tuo Zhao.
\newblock Why do deep residual networks generalize better than deep feedforward
  networks?--a neural tangent kernel perspective.
\newblock \emph{arXiv preprint arXiv:2002.06262}, 2020.

\bibitem[Jacot et~al.(2018)Jacot, Gabriel, and Hongler]{jacot2018neural}
Arthur Jacot, Franck Gabriel, and Cl{\'e}ment Hongler.
\newblock Neural tangent kernel: Convergence and generalization in neural
  networks.
\newblock In \emph{Advances in neural information processing systems}, pp.\
  8571--8580, 2018.

\bibitem[Jacot et~al.(2019)Jacot, Gabriel, and Hongler]{jacot2019freeze}
Arthur Jacot, Franck Gabriel, and Cl{\'e}ment Hongler.
\newblock Freeze and chaos for dnns: an ntk view of batch normalization,
  checkerboard and boundary effects.
\newblock \emph{arXiv preprint arXiv:1907.05715}, 2019.

\bibitem[Lee et~al.(2017)Lee, Bahri, Novak, Schoenholz, Pennington, and
  Sohl-Dickstein]{lee2017deep}
Jaehoon Lee, Yasaman Bahri, Roman Novak, Samuel~S Schoenholz, Jeffrey
  Pennington, and Jascha Sohl-Dickstein.
\newblock Deep neural networks as gaussian processes.
\newblock \emph{arXiv preprint arXiv:1711.00165}, 2017.

\bibitem[Lee et~al.(2020)Lee, Schoenholz, Pennington, Adlam, Xiao, Novak, and
  Sohl-Dickstein]{lee2020finite}
Jaehoon Lee, Samuel~S Schoenholz, Jeffrey Pennington, Ben Adlam, Lechao Xiao,
  Roman Novak, and Jascha Sohl-Dickstein.
\newblock Finite versus infinite neural networks: an empirical study.
\newblock \emph{arXiv preprint arXiv:2007.15801}, 2020.

\bibitem[Li et~al.(2021)Li, Nica, and Roy]{li2021future}
Mufan~Bill Li, Mihai Nica, and Daniel~M Roy.
\newblock The future is log-gaussian: Resnets and their
  infinite-depth-and-width limit at initialization.
\newblock \emph{arXiv preprint arXiv:2106.04013}, 2021.

\bibitem[Matthews et~al.(2018)Matthews, Rowland, Hron, Turner, and
  Ghahramani]{matthews2018gaussian}
Alexander G de~G Matthews, Mark Rowland, Jiri Hron, Richard~E Turner, and
  Zoubin Ghahramani.
\newblock Gaussian process behaviour in wide deep neural networks.
\newblock \emph{arXiv preprint arXiv:1804.11271}, 2018.

\bibitem[Neal(1996)]{neal1996priors}
Radford~M Neal.
\newblock Priors for infinite networks.
\newblock In \emph{Bayesian Learning for Neural Networks}, pp.\  29--53.
  Springer, 1996.

\bibitem[Novak et~al.(2018)Novak, Xiao, Lee, Bahri, Yang, Hron, Abolafia,
  Pennington, and Sohl-Dickstein]{novak2018bayesian}
Roman Novak, Lechao Xiao, Jaehoon Lee, Yasaman Bahri, Greg Yang, Jiri Hron,
  Daniel~A Abolafia, Jeffrey Pennington, and Jascha Sohl-Dickstein.
\newblock Bayesian deep convolutional networks with many channels are gaussian
  processes.
\newblock \emph{arXiv preprint arXiv:1810.05148}, 2018.

\bibitem[Pineda(1988)]{pineda1988generalization}
Fernando~J Pineda.
\newblock Generalization of back propagation to recurrent and higher order
  neural networks.
\newblock In \emph{Neural information processing systems}, pp.\  602--611,
  1988.

\bibitem[Poole et~al.(2016)Poole, Lahiri, Raghu, Sohl-Dickstein, and
  Ganguli]{poole2016exponential}
Ben Poole, Subhaneil Lahiri, Maithra Raghu, Jascha Sohl-Dickstein, and Surya
  Ganguli.
\newblock Exponential expressivity in deep neural networks through transient
  chaos.
\newblock In \emph{Advances in neural information processing systems}, pp.\
  3360--3368, 2016.

\bibitem[Powell(1970)]{powell1970hybrid}
Michael~JD Powell.
\newblock A hybrid method for nonlinear equations.
\newblock \emph{Numerical methods for nonlinear algebraic equations}, 1970.

\bibitem[Raghu et~al.(2017)Raghu, Poole, Kleinberg, Ganguli, and
  Sohl-Dickstein]{raghu2017expressive}
Maithra Raghu, Ben Poole, Jon Kleinberg, Surya Ganguli, and Jascha
  Sohl-Dickstein.
\newblock On the expressive power of deep neural networks.
\newblock In \emph{international conference on machine learning}, pp.\
  2847--2854. PMLR, 2017.

\bibitem[Schoenholz et~al.(2016)Schoenholz, Gilmer, Ganguli, and
  Sohl-Dickstein]{schoenholz2016deep}
Samuel~S Schoenholz, Justin Gilmer, Surya Ganguli, and Jascha Sohl-Dickstein.
\newblock Deep information propagation.
\newblock \emph{arXiv preprint arXiv:1611.01232}, 2016.

\bibitem[Vershynin(2019)]{vershynin2019high}
Roman Vershynin.
\newblock High-dimensional probability, 2019.

\bibitem[Winston \& Kolter(2020)Winston and Kolter]{winston2020monotone}
Ezra Winston and J~Zico Kolter.
\newblock Monotone operator equilibrium networks.
\newblock \emph{arXiv preprint arXiv:2006.08591}, 2020.

\bibitem[Xiao et~al.(2018)Xiao, Bahri, Sohl-Dickstein, Schoenholz, and
  Pennington]{xiao2018dynamical}
Lechao Xiao, Yasaman Bahri, Jascha Sohl-Dickstein, Samuel~S Schoenholz, and
  Jeffrey Pennington.
\newblock Dynamical isometry and a mean field theory of cnns: How to train
  10,000-layer vanilla convolutional neural networks.
\newblock \emph{arXiv preprint arXiv:1806.05393}, 2018.

\bibitem[Yang(2019)]{yang2019tensor}
Greg Yang.
\newblock Tensor programs i: Wide feedforward or recurrent neural networks of
  any architecture are gaussian processes.
\newblock \emph{arXiv preprint arXiv:1910.12478}, 2019.

\bibitem[Yang(2020)]{yang2020tensor}
Greg Yang.
\newblock Tensor programs ii: Neural tangent kernel for any architecture.
\newblock \emph{arXiv preprint arXiv:2006.14548}, 2020.

\end{thebibliography}
\bibliographystyle{iclr2023_conference}

\appendix
\onecolumn
\section{Formal derivation of weight-tied network}
In this section we formally derive the NTK of a DEQ (weight-tied) model, and show that they converge to the same limit as derived in \cref{sec:deqntk}. The argument is nearly identical to that of \citet{alemohammad2020recurrent}, which heavily depends on the \textsc{Nester$\top$} program \citep{yang2020tensor}. We will first give a brief introduction, and then adapt to our setting.

\begin{definition}
	\textsc{Nester$\top$} program is a program (as in type system) of which the variables take three-types: $\tA$-vars, $\tG$-vars, and $\tH$-vars. Any variables are generated by one of the rules in $\matmul$ (matrix multiplication), $\nonlin$ (nonlinearity), $\lincomb$ (linear combination), or $\trsp$ (matrix transpose). We also sometimes explicitly express the dimensionality of a variable in the following way:
	\begin{itemize}
		\item If $x\in\R^n$, and is of type $\tG, \tH$, we write $x:\tG(n)$ or $x:\tH(n)$.
		\item If $A\in\R^{n\times m}$, we write $A:\tA(n, m)$.
	\end{itemize}
	
	The program goes as following:
	\paragraph{Input} A set of $\tG$-vars and $\tA$-vars.
	\paragraph{Body} Any variable is introduced by the following rules:
	\begin{itemize}
		\item $\trsp$. If $A:\tA(n,m)$, then $A^\top:\tA(m,n)$.
		\item $\matmul$. If $A:\tA(n,m)$ and $x:\tH(m)$, then $Ax:\tG(n)$.
		\item $\lincomb$. If $g^1,\ldots, g^k:\tG(n)$ and $a^1,\ldots, a^k\in\R$, then $\sum_{i=1}^ka^ig^i: \tG(n)$.
		\item $\nonlin$. If $x^1,\ldots, x^k:\tG(n)$, and $\phi:\R^k\to\R$ is a coordinate-wise nonlinear function, then $\phi(x^1,\ldots, x^k):\tH(n)$.
	\end{itemize}
	
	\paragraph{Output} The program outputs a scalar of the form
	\[
		\frac{1}{n} \sum_{\alpha=1}^n \psi\left(h_\alpha^1, \ldots, h_\alpha^k\right)
	\]
	for $h^1\ldots h^k:\tH(n)$.
\end{definition}

For example, a depth-$d$ approximation to a DEQ model is provided in \cref{alg:deqnester}. For simplicity, we left out the scaling $\sigma_W^2/\sqrt{n}$ (as was done in \cite{yang2020tensor}).
\begin{algorithm}
\caption{\textbf{\textsc{Nester$\top$} program} Depth-$d$ approximation to a DEQ model }\label{alg:deqnester}
\begin{algorithmic}
\Require $Ux, Uy:\tG(n), W:\tA(n,n), b:\tG(n), v:\tG(n)$. Polynomially-bounded coordinate-wise nonlinear function $\phi$.
\For {$h=1,\ldots, d$}
	\For {$z\in\{x, y\}$}
		\State 
			$
				f^{(h)}(z) = Wg^{(h-1)}(z)+Uz+b : \tG(n).
			$
		\State $g^{(h)}(z)=\phi(f^{(h)}(z)):\tH(n)$.
		\State $//$ The network outputs $f^{(d+1)}(z):=\frac{v^\top g^{(d)}(z)}{n}$, but we don't express this in the program.
		\State $//$ Backprop, for varible $u$, let $du:=\sqrt{n}\nabla_uf^{(d+1)}(z)$.
		\State $dg^{(d)}(z) = v:\tG(n)$.
		\State $df^{(d)}(z) = \phi'(f^{(d)}(z))\odot dg^{(d)}(z):\tH(n)$. \Comment{We use $\odot$ for Hadamard product.}
		\State $dg^{(h)}(z) = W^\top df^{(h+1)}(z):\tG(n)$.
		\State $df^{(h)}(z) = \phi'(f^{(h)}(z))\odot dg^{(h)}(z):\tH(n)$.
	\EndFor
\EndFor

\end{algorithmic}
\end{algorithm}

One can express many neural network architectures into a \textsc{Nester$\top$} program, but not all. The required regularity condition is the so-called \textit{BP-like}:
\begin{definition}[BP-like]
\label{def:bplike}
	A \textsc{Nester$\top$} program is $\textit{BP-like}$ if there exists a non-empty set of input $\tG(n)$-vars $v^1,\ldots, v^k$ s.t:
	\begin{enumerate}
		\item If $W^\top z$ is used in the program for some $z:\tH(n)$, and $W:\tA(n,m)$ is an input $\tA$-var, then $z$ must be an odd function of $v^1,\ldots, v^k$. That is,
		\[
			z\left(-v^1, \ldots,-v^k \text {, all other $\tG$-vars }\right)=-z\left(v^1, \ldots, v^k \text {, all other $\tG$-vars }\right).
		\]
		\item If $Wz$ is used in the program for some $z:\tH(m)$, and $W:\tA(n,m)$ is an input $\tA$-var, then $z$ cannot depend on any of $v^1,\ldots v^k$.
		\item $v^1,\ldots, v^k$ are sampled with zero mean and independently from all other $\tG$-vars.
	\end{enumerate}
\end{definition}

\begin{definition}[Polynomially-bounded]
	We say a function $f:\R^k\to\R$ is polynomially-bounded if $|\phi(x)|\leq C\|x\|^p+c$ for some $c, C, p>0$, for all $x\in\R^k$. Note that ReLU and inner product are polynoimially-bounded.
\end{definition}

Recall that the simple gradient independence assumption (GIA) check we give in \cref{sec:deqntk}:
\begin{condition}[Simple GIA check]
	\textit{Gradient independence assumption} is a heuristic that for any matrix $W$, we assume $W^\top$ used in  backprop is independet from $W$ used in the forward pass. We can regard this assumption holds in the NTK computation if the following simple check holds: 
	the output layer is sampled independently with zero mean from all other parameters and it not used anywhere else in the interior of the network, that is, if the output of the network is $v^\top x$, then $v$ is independent of $x$.
\end{condition}
Apparently our DEQ formulation satisfy the simple GIA check, notice that by formulation, the second and third condition in \cref{def:bplike} are trivially satisfied. Also since $v$ is the last layer weight, any $\tG$-var of the form $W^\top z$ only shows up in the backpropogation, and is linear (thus odd) in $v$ as well. Hence the first condition is also satisified. So any network structure that satisfies the simple GIA check is automatically BP-like.

\begin{setup}\label{setup:bplike}
	For \textsc{Nester$\top$} program, we assume that each entry in $W:\tA(n,m)$ is sampled from $\cN(0, \sigma_W^2/n)$, and any input $\tG$-vars $x\sim\cN(\mu^{in}, \Sigma^{in})$. We remark that this does not contradict with the parameterization that we mentioned in the main text where the entries of input $\tA$-vars $W, U$ are standard Gaussians. One just needs to properly scale their variables.
\end{setup}

\begin{theorem}[BP-like \textsc{Nester$\top$} program Master theorem]
	Fix any $\textit{BP-like}$ \textsc{Nester$\top$} program that satisfies \cref{setup:bplike}, and all its nonlinearities are polynomially-bounded. If $g^1,\ldots, g^M$ are all $\tG$-vars in the program, then for any polynomially-bounded $\psi:\R^M\to \R$, as $n\to\infty$, we have
	\[
	\frac{1}{n} \sum_{\alpha=1}^n \psi\left(g_\alpha^1, \ldots, g_\alpha^M\right) \stackrel{\text { a.s. }}{\longrightarrow} \underset{Z \sim \mathcal{N}(\mu, \Sigma)}{\mathbb{E}} \psi(Z)=\underset{Z \sim \mathcal{N}(\mu, \Sigma)}{\mathbb{E}} \psi\left(Z^{g^1}, \ldots, Z^{g^M}\right),
	\]
	where $Z=\{Z^{g^1}, \ldots, Z^{g^M}\}\in\R^M$, $\mu=\{\mu(g^i)\}_{i\in[M]}\in\R^M$, $\Sigma=\{\Sigma(g^i, g^j)\}_{i,j=1}^M\in\R^{M\times M}$ are given by
	\begin{equation}
	\begin{aligned}
	\mu(g) &=\left\{\begin{array}{ll}
	\mu^{\text {in }}(g) & \text { if } g \text { is input, } \\
	\sum_{i=1}^k a^i\mu(g^i) & \text{if $g=\sum_{i=1}^k a^i g^i$}\\
	0 & \text { otherwise } 
	\end{array}\right.\\
	\Sigma(g, \bar{g}) &= \begin{cases}\Sigma^{\text {in }}\left(g, g^{\prime}\right) & \text { if } g, g^{\prime} \text { are inputs } \\
	\sum_{i=1}^k a^i\Sigma(g^i, \bar g) & \text{if $g=\sum_{i=1}^k a^i g^i$}\\
	\sum_{i=1}^k a^i\Sigma(g, \bar g^i) & \text{if $\bar g=\sum_{i=1}^k a^i \bar g^i$}\\
	\sigma_W^2 \mathbb{E}_Z \phi(Z) \bar{\phi}(Z) & \text { if } g=W h, \bar{g}=W \bar{h}, \\
	0 & \text { otherwise. }\end{cases}
	\end{aligned}
	\end{equation}
\end{theorem}

We are now equipped to derive the NTK of a depth-$d$ approximation to a DEQ. Particularly, we have
\begin{align*}
\begin{split}
	\nabla_W f^{(d+1)}(x) &= \frac{\sigma_W}{ n}\sum_{h=1}^d df^{(h)}g^{(h-1)}(x)^\top,
\end{split}
\end{align*}
hence 
\begin{align*}
\begin{split}
	\ip{\nabla_W f^{(d+1)}(x)}{\nabla_W f^{(d+1)}(y)} &= \sigma_W^2\sum_{l, h=1}^d \frac{{df^{(h)}(x)}^\top df^{(l)}(y)}{n}\frac{g^{(h-1)}(x)^\top g^{(l-1)}(y)}{n}.\\
\end{split}
\end{align*}
From this point, we need to calculate 
\[
	\E_W\left[{df^{(h)}(x)}^\top df^{(l)}(y)\right] \text{ and } \E_W\left[g^{(h-1)}(x)^\top g^{(l-1)}(y)\right].
\]
In the end, applying the Master theorem with $\psi(x,y)=x^\top y$ on $\frac{{df^{(h)}}^\top df^{(l)}}{n}$ and $\frac{g^{(h-1)}(x)^\top g^{(l-1)}(y)}{n}$ shows that these empirical averages converge to the expectations.

\begin{remark}
	Notice that the Master theorem talks about $\tG$-vars, while $df^{(h)}$ and $g^{(h)}$ are $\tH$-vars. We can always compose $\psi' = \psi\circ\phi$, where $\psi$ is the inner product and $\phi$ is coordinate-wise nonlinearity (such as ReLU), and apply the Master theorem on $\psi'$, as long as it is still polynomially-bounded. 
\end{remark}

\begin{align*}
\begin{split}
	&\E_W\left[{df^{(h)}(x)}^\top df^{(l)}(y)\right] = \E\left[ \p{\phi'(f^{(h)}(x))\odot dg^{(h)}(x)}^\top\p{\phi'(f^{(l)}(y))\odot dg^{(l)}(y)}\right]\\
	&=\E\left[ \phi'(f^{(h)}(x))^\top\phi'(f^{(l)}(y))\cdot (dg^{(h)}(x)^\top dg^{(l)}(y))\right]\\
	&=\underbrace{\E\left[ \phi'(f^{(h)}(x))^\top\phi'(f^{(l)}(y))\right]}_\text{A}\cdot\underbrace{\E\left[ (dg^{(h)}(x)^\top dg^{(l)}(y))\right]}_\text{B}.
\end{split}
\end{align*}

By the Master theorem and GIA, $\phi'(f^{(h)})$ and $dg^{(h)}$ are introduced by different $\tA$-vars ($W$ and $W^\top$), hence their coviance is $0$. This justifies the last step above.

When $h, l<d$, by the Master theorem we have
\[
	B=\sigma_W^2\E[df^{(h+1)}(x)^\top df^{(l+1)}(y)].
\]
Notice that this gives a recursive expression, WLOG we assume that $h<l$, this induction will lead to
\[
	\E[df^{(h+t)}(x)^\top df^{(d)}(y)]=\E\left[\p{\phi'(f^{(h+t)}(x))\odot dg^{(h+t)}(x)}^\top \p{\phi'(f^{(d)}(y))\odot v}\right]=0,
\]
for some $t>0$. The reason why this is zero is still due to the Master theorem, as $df^{(h+t)}(x)$ and $df^{(d)}(y)$ are $\tG$-vars involved with different $\tA$-vars $W$ and $v$.

This shows that when $h\neq l$, $\E_W\left[{df^{(h)}(x)}^\top df^{(l)}(y)\right]=0$. Hence we only have to consider the case $h=l$. By the Master theorem we have
\[
 A= \E_{u,v}\left[\phi'(u)\phi'(v)\right], \E_W\left[g^{(h)}(x)^\top g^{(h)}(y)\right]=\E_{u,v}\left[\phi(u)\phi(v)\right],
 \] 
 where
 \[(u,v)\sim\cN\p{0, \begin{pmatrix}
 	\Sigma^{(h-1)}(x,x) & \Sigma^{(h-1)}(x,y)\\
 	\Sigma^{(h-1)}(y,x) & \Sigma^{(h-1)}(y,y)
 \end{pmatrix}}.
\]
Notice this exactly recovers the calculation of NTK when the weights are un-tied. The exact same argument can be applied to $\nabla_Uf$ and $\nabla_bf$. Since such equivalence holds for all depth $d$, it also holds in the limit of $d\to\infty$.

\paragraph{Key takeaway} The $\textsc{Nester$\top$}$ program allows us to calculate the NTK of a weight-tied network in exactly the same way as the weight-untied network.

\section{Details of \Cref{sec:deqntk}}\label{sec:deqntk_appendix}
In this section, we give the detailed derivation of DEQ-of-NTK. There are two terms that are different from NTK: $\Sigma^{(h)}(x,y)$ and the extra $\E_{\theta}\left[ \ip{\frac{\partial f(\theta ,x)}{\partial U} }{\frac{\partial f(\theta ,y)}{\partial U} }\right]$ in the kernel.

Let us restate the depth-$d$ approximation to DEQs here:

Let $m$ be the input dimension, $x, y\in\R^{m}$ be a pair of inputs, $n$ be the width of the $h^{th}$ hidden layers. Define the depth-$d$ approximation to DEQ as follows:
\begin{align*}
	& f_\theta^{(h)}(x) = \sqrt{\frac{\sigma_W^2}{n}} {W^{(h)}}{g^{(h-1)}(x)}+\sqrt{\frac{\sigma_U^2}{n}}{U^{(h)}}{x}+\sqrt{\frac{\sigma_b^2}{n}} b^{(h)}, \ h\in[L]\\
	& g^{(d)}(x) = \sigma(f^{(L)}_\theta(x))\\
	& f^{(d+1)}(x) = \sigma_v^2\cdot v^T g^{(d+1)}_\theta(x)
\end{align*}
where $W^{(h)}\in\R^{n\times n}$, $U^{(h)}\in\R^{n\times m}$, and $v\in\R^n$ are the internal weights, and $b^{(h)}\in\R^{n}$ are the bias terms. These parameters are chosen using the NTK initialization. Let us pick $\sigma_W,\sigma_U,\sigma_b\in\R$ arbitrarily in this section.

\deqntk*
\begin{proof}[Proof of \cref{thm:deqntk}]
	First we note that 
	\begin{align*}
		&\mathbb{E}\left[\left[{f}^{(h+1)}({x})\right]_{i} \cdot\left[{f}^{(h+1)}\left(y\right)\right]_{i} \mid {f}^{(h)}\right] \\
		=&\frac{\sigma_W^2}{n} \sum_{j=1}^{n} \sigma\left(\left[{f}^{(h)}({x})\right]_{j}\right) \sigma\left(\left[{f}^{(h)}\left(y\right)\right]_{j}\right) + \frac{\sigma_U^2}{n}\sum_{j=1}^n x^\top y + \sigma_b^2\\
		\to & \Sigma^{(h+1)}(x, y)\ a.s
	\end{align*}
	where the first line is by expansion the original expression and using the fact that $W, U, b$ are all independent. The last line is from the strong law of large numbers. This shows how the covariance changes as depth increases with input injection. 
	
	Recall the splitting:
	\begin{align*}
			&\Theta^{(L)}(x,y) = \E_{\theta}\left[ \ip{\frac{\partial f(\theta ,x)}{\partial \theta} }{\frac{\partial f(\theta ,y)}{\partial \theta} }\right]\\
			=&\underbrace{\E_{\theta}\left[ \ip{\frac{\partial f(\theta ,x)}{\partial W} }{\frac{\partial f(\theta ,y)}{\partial W} }\right]}_\text{$\circled{1}$}+\underbrace{\E_{\theta}\left[ \ip{\frac{\partial f(\theta ,x)}{\partial U} }{\frac{\partial f(\theta ,y)}{\partial U} }\right]}_\text{$\circled{2}$}\\
			&\qquad +\underbrace{\E_{\theta}\left[ \ip{\frac{\partial f(\theta ,x)}{\partial b} }{\frac{\partial f(\theta ,y)}{\partial b} }\right]}_\text{$\circled{3}$} 
			+\underbrace{\E_{\theta}\left[ \ip{\frac{\partial f(\theta ,x)}{\partial v} }{\frac{\partial f(\theta ,y)}{\partial v} }\right]}_\text{$\circled{4}$}
	\end{align*}
	
	The following equation has been proven in many places:
	\[
		\circled{1} = \sum_{h=1}^{d+1}\left(\sigma_W^2\mathop\E_{(u, v)\sim\cN(0, \Lambda^{(h)})}[\sigma(u)\sigma(v)] \cdot \prod_{h^{\prime}=h}^{d+1} \dot{\Sigma}^{\left(h^{\prime}\right)}\left(x, y\right)\right),\ 		\circled{3} = \sum_{h=1}^{d+1}\left(\sigma_b^2 \cdot \prod_{h^{\prime}=h}^{d+1} \dot{\Sigma}^{\left(h^{\prime}\right)}\left(x, y\right)\right),
	\]
	and
	$
		\circled{4}=\sigma_v^2\mathop\E_{(u, v)\sim\cN(0, \Lambda^{(h)})}[\sigma(u)\sigma(v)].
	$
	For instance, see \citet{arora2019exact}. So we only need to deal with the second term $\E_{\theta}\left[ \ip{\frac{\partial f(\theta ,x)}{\partial U} }{\frac{\partial f(\theta ,y)}{\partial U} }\right]$. Write $f=f_\theta(x)$ and $\tilde f = f_\theta(y)$, by chain rule, we have
	\begin{align*}
		&\ip{\frac{\partial f}{\partial U^{(h)}}}{\frac{\partial\tilde f}{\partial U^{(h)}}}\\
		=&\ip{\frac{\partial f}{\partial f^{(h)}}\frac{\partial f^{(h)})}{\partial U^{(h)}}}{\frac{\partial \tilde f}{\partial \tilde f^{(h)}}\frac{\partial \tilde f^{(h)})}{\partial U^{(h)}}}\\
		=& \ip{  \frac{\partial f^{(h)}}{\partial U^{(h)}}  }{  \frac{\partial\tilde f^{(h)}}{\partial U^{(h)}}  }\cdot \ip{       \frac{\partial f}{\partial f^{(h)}}   }{ \frac{\partial \tilde f}{\partial \tilde f^{(h)}}  }\\
		\to & \sigma_U^2x^\top y \cdot \prod_{h'=h}^{d+1}\dot\Sigma^{(h')}(x,y) 
	\end{align*}
	where the last line uses the existing conclusion that  $\ip{       \frac{\partial f}{\partial f^{(h)}}   }{ \frac{\partial \tilde f}{\partial \tilde f^{(h)}}  }\to \prod_{h'=h}^{d+1}\dot\Sigma^{(h')}(x,y)$, this convergence almost surely holds when $N\to\infty$ by law of large numbers.
	
	Finally, summing $\ip{\frac{\partial f}{\partial U^{(h)}}}{\frac{\partial \tilde f}{\partial U^{(h)}}}$ over $h\in[d]$ we conclude the assertion.
\end{proof}

\begin{lemma}\label{lma:deqntkconverge}
	Use the same notations and settings in \cref{thm:deqntk}. With input data $x,y\in\S^{d-1}$, parameters $\sigma_W^2,\sigma_U^2,\sigma_b^2$ following the DEQ-NTK initialization, $\Theta^{(d)}(x,y)$ in \cref{eq:deqntk} converges absolutely if $\sigma_W^2<1$.
\end{lemma}
 
\begin{proof}
	Since we pick $x,y\in\S^{d-1}$, and by DEQ-NTK initialization, we always have $\Sigma^{(h)}(x,y) < 1$ for $x\neq y$. Let $\rho=\Sigma^{(h)}(x,y)$, by \cref{eq:deqntkdotsigma} and \cref{eq:dualactivationdotsigma}, if $\sigma_W^2<1$, then there exists $c$ such that $\dot\Sigma^{(h)}(x,y)<c<1$ for finite number of pairs $x\neq y$ on $\S^{d-1}$, and large enough $h$. This is because $\lim_{h\to\infty}\dot\Sigma^{(h)}(x,y)=\dot\Sigma^*(x,y)<\dot\Sigma^*(x,x)<1$.
		
	Use comparison test,
	\begin{align*}
		 	\lim_{L\to\infty }\sum_{h=1}^{L+1}\left| \left(\Sigma^{(h-1)}\left(x, y\right)\right) \cdot \prod_{h^{\prime}=h}^{L+1} \dot{\Sigma}^{\left(h^{\prime}\right)}\left(x, y\right)\right|
		 	<1+ \lim_{L\to\infty }\sum_{h=1}^{L+1} c^{L+1-h}.
	\end{align*}
	Since $c<1$, the geometric sum converges absolutely, hence $\Theta^{(d)}(x,y)$ converges absolutely if $\sigma_W^2<1$, and the limit exists.
\end{proof}

\depntkconverge*
\begin{proof}[Proof of \Cref{thm:depntkconverge}]

Due to the fact that $x\in\S^{d-1}$, $\sigma$ being normalized, and DEQ-NTK initialization, one can easily calculate by induction that for all $h\in[L]$:
$
	\Sigma^{(h)}(x,x)=\sigma_W^2 \mathop \E_{u\sim\cN(0,1)}[\sigma(u)^2]+\sigma^2_Vx^\top x +\sigma^2_b =1
$
This indicates that in \cref{eq:deqntklambda}, the covariance matrix has a special structure $
	\Lambda^{(h)}(x,y) = \begin{pmatrix}
 1 & \rho \\ \rho & 1	
 \end{pmatrix}
$, where $\rho=\Sigma^{(h-1)}(x,y)$ depends on $h, x, y$. For simplicity we omit the $h, x, y$ in $\Lambda^{(h)}(x,y)$.  As shown in \citet{cho2009kernel}:
\begin{align}
	& \mathop{\E}_{(u, v)\sim\cN\left(0, \Lambda\right)
}[\sigma(u)\sigma(v)]=\frac{\sqrt{1-\rho^{2}}+\left(\pi-\cos ^{-1}(\rho)\right) \rho}{\pi}  \label{eq:dualactivationsigma} \\ 
	& \mathop{\E}_{(u, v)\sim\cN\left(0, \Lambda\right)
}[\dot\sigma(u)\dot\sigma(v)]=\frac{\pi-\cos ^{-1}(\rho)}{\pi} \label{eq:dualactivationdotsigma} 
\end{align}

Adding input injection and bias, we derive \cref{eq:dualactivationsigmawithinjection} from \cref{eq:dualactivationsigma}, and similarly, \cref{eq:dualactivationdotsigmawithinjection} from \cref{eq:dualactivationdotsigma}. 
Notice that iterating \crefrange{eq:deqntksigmainit}{eq:deqntksigma} to solve for $\Sigma^{(h)}(x,y)$ is equivalent to iterating $(R_\sigma \circ\cdots \circ R_\sigma)(\rho)$ with initial input $\rho=x^\top y$. Take the derivative
\begin{align*}
	\abs{\frac{dR_\sigma(\rho)}{d\rho} }= \abs{\sigma_W^2 \p{1-\frac{\cos ^{-1}(\rho)}{\pi}}}<1,  \text{ if $\sigma_W^2<1$ and $-1\leq\rho<1$}.
\end{align*}
For $x\neq y$ we have $-1\leq\rho<c<1$ for some $c$ (this is because we only have finite number of inputs $x,y$) and by DEQ-NTK initialization we have $\sigma_W^2<1$, so the above inequality hold. Hence $R_\sigma(\rho)$ is a contraction on $[0,c]$, and we conclude that the fixed point $\rho^*$ is attractive.

By \cref{lma:deqntkconverge}, if $\sigma_W^2<1$, then the limit of \cref{eq:deqntk} exists, so we can rewrite the summation form in \cref{eq:deqntk} in a recursive form:
\begin{align*}
	&\Theta^{(0)}(x,y)=\Sigma^{(0)}(x,y),\\
	& \Theta^{(d+1)}(x,y)=\dot\Sigma^{(d+1)}(x,y)\cdot \Theta^{(d)}(x,y)+ \Sigma^{(d+1)}(x,y).
\end{align*}
Directly solve the fixed point iteration for the internal representation:
\begin{align}\label{eq:fixpointcalculationdeq}
\begin{split}
	&\lim_{d\to\infty}\Theta^{(d+1)}(x,y)\\
	&=\lim_{d\to\infty}\p{\dot\Sigma^{(d+1)}(x,y)\cdot \Theta^{(d)}(x,y)+ \Sigma^{(d+1)}(x,y)}\\
	\Longrightarrow &\lim_{L\to\infty}\Theta^{(d+1)}(x,y)\\
	&=\dot\Sigma^{*}(x,y)\cdot \lim_{d\to\infty}\Theta^{(d)}(x,y)+ \Sigma^{*}(x,y)\\
	\Longrightarrow &\lim_{d\to\infty}\Theta^{(d)}(x,y)\\
	&=\dot\Sigma^{*}(x,y)\cdot \lim_{d\to\infty}\Theta^{(d)}(x,y)+ \Sigma^{*}(x,y).
\end{split}
\end{align}
Solving for $\lim_{d\to\infty} \Theta^{(d)}(x, y) $ we get $\Theta^*(x,y) = \frac{ \Sigma^*(x,y)}{1-  \dot\Sigma^*(x,y)}.$ Finally, we process the classification layer and get $\Theta=\dot\Sigma\cdot\Theta^*+\Sigma$, where $\dot\Sigma=\sigma_v^2\dot\rho^*$ and $\Sigma=\sigma_v^2\rho^*$.This concludes the proof

\end{proof}

\subsection{DEQ-of-NTK vs. NTK-of-DEQ}
In this section we discuss \cref{thm:exchangelimit} in detail. Recall that the NTK is the kernel matrix formed by an infinitely-wide network. To be more precisely, if the network has depth $d$, then
\[
	\Theta^{(d)}(x,y)=\E_\theta\left[\ip{\frac{\partial f(\theta,x)}{\partial\theta}}{\frac{\partial f(\theta,y)}{\partial\theta}}\right].
\] 
It is straightforward to define its width-$n$ approximation:
\[
	\Theta_{n}^{(d)} = \sum_{h=1}^d\ip{\frac{\partial f(\theta,x)}{\partial\theta^{(h)}}}{\frac{\partial f(\theta,y)}{\partial\theta^{(h)}}},
\]
where $\theta^{(h)}$ is the parameter of the $h$th layer with width $n$. The name of $\lim_{d\to\infty}\lim_{n\to\infty}\Theta_{n}^{(d)}$ being the DEQ of NTK is intuitive: because we first \footnote{Here by ``first'' we meant the order when you calculate the limits: you first fix $d$ and take the limit of $n$. Not the actual order from left to right.} bring width to infinity, that is, the NTK is first derived. Then we talk about the NTK's infinite-depth limit. This is in distinction to our desired quantity, $\lim_{n\to\infty}\lim_{d\to\infty}\Theta_{n}^{(d)}$, which is the NTK of DEQ naturally. In this section we show they are indeed equivalent under certain conditions.

First we introduce some notations. Consider a finite depth iteration of a NTK with depth $d+1$, and for simplicity let the bias term $b^{(h)}=0$ for all $h\in[d+1]$. A straightforward calculation show that 
\begin{align*}
	&\text{For } h\in[L+1]: \frac{df(\theta, x)}{dW^{(h)}} = p^{(h)}(x)\p{g^{(h-1)}(x)}^\top\\
	&\qquad\qquad\qquad\quad\ \frac{df(\theta, x)}{dU^{(h)}} = p^{(h)}(x)\cdot x^\top\\
	&\text{where } p^{(h)}(x)=\begin{cases}
		1\in\R^n, & h=d+1\\
		\sqrt{\frac{\sigma_W^2}{N_h}}\operatorname{diag}\p{\dot\sigma\p{f^{(h)}(x)}}\p{W^{(h+1)}}^\top p^{(h+1)}(x) & h\leq d
	\end{cases}
\end{align*}
Here $\operatorname{diag}\p{\dot\sigma\p{f^{(h)}(x)}}\in\R^{N_h\times N_h}$. Let $N_h=n$ for all $h$, and $W^{(h+1)}:=v$.  Notice that
\[
	\operatorname{diag}\p{\dot\sigma\p{f^{(h)}(x)}}\p{W^{(h+1)}}^\top p^{(h+1)}(x) = \dot\sigma\p{f^{(h)}(x)}\odot\p{\p{W^{(h+1)}}^\top p^{(h+1)}(x)},
\]
and we use these terms interchangeably. For simplicity, we omit all the $x$ in the terms and write $f^{(h)}:=f^{(h)}(x)$, etc. Write $\dot\sigma^{(h)} =\dot\sigma\p{f^{(h)}(x)} $. Notice that applying $\sigma(\cdot)$ or Hadamard product with $\dot\sigma^{(h)}$ only decreases norms.

\begin{lemma}[Probablisitc Moore-Osgood for double sequence]\label{lma:mo}
	Let $a_{n, d}$ be a random double sequence in a complete space. Assume for any $\epsilon>0, \delta\in(0, 1)$, there exists $N(\delta)>0$ and $D(\epsilon)>0$ such that for all $n>N$ and $d>D$, with probability at least $1-\delta$ we have $|a_{n,d}-a_n|<\epsilon$ (we may refer to this property as uniform convergence with high probability). And for any $d\in\mathbb N$ we have $\lim_{n\to\infty} a_{n, d}=a_d$ almost surely, then with high probability:
	\[
		\lim_{n\to\infty} \lim_{d\to\infty} a_{n, d}=\lim_{d\to\infty} \lim_{n\to\infty} a_{n, d}.
	\]  
\end{lemma}
\begin{proof}
	We sometimes also write $a_d(n)$ to stress that we consider the sequence as a function of $n$. By assumption, for any $\delta\in(0, 1), \epsilon>0$, there exists $N, D$ such that for all $n>N$, $d, e>D$, $|a_d(n)-a_e(n)|<\epsilon$ with probability at least $1-\delta$. Since here $N$ does not depend on $D$, let $n\to\infty$ we get the following statement holds almost surely:
	\[
		d, e>D\Longrightarrow |a_d-a_e|<\epsilon \text{ with probability at least $1-\delta$}.
	\]
	This shows that $a_d:=\lim_{n\to\infty}a_{n, d}$ is a Cauchy sequence and have a finite limit $\lim_{d\to\infty}a_d=L$.
	
	Now define $a(n):=a_n=\lim_{d\to\infty}a_{n,d}$, for $d>D(\epsilon)$:
	\[
		\abs{a(n)-L}\leq \underbrace{|a(n)-a_d(n)|}_\text{A}+\underbrace{|a_d(n)-a_d|}_\text{B}+\underbrace{|a_d-L|}_\text{C}.
	\]
	By assumption, pick large enough $n$, we have $A<\epsilon$ with probability at least $1-\delta$. By the Cauchy sequence argument above, we have $C<\epsilon$ with high probability. Finally since $a_d(n)\to a_d$ pointwise for every $d$, we can choose $n$ large enough such that $B<\epsilon$. This concludes our proof.
\end{proof}

	We want to remark that the above \cref{lma:mo} relies on a more general notion of ``conditional almost sure convergence''. In particular, we only assume that $|a_{n,d}-a_n|<\epsilon$ almost surely conditioned on an event with probability at least $1-\delta$:
	\[
		P\p{\lim_{d\to\infty} a_{n,d}=a_n \big | E}=1,\text{ where $P(E)>1-\delta$ for all large enough $n$}.
	\]
	Notice here we are not explicit about how $\delta$ evolves with $n$. When we use this lemma in \cref{thm:exchangelimit}, we have $\delta=o(n)$ which will instead gives us a convergence in probability result. To be complete, we also provide the weaker result and its proof here.

	\begin{lemma}[Another probablisitc Moore-Osgood for double sequence]
		Let $a_{n, d}$ be a random double sequence in a complete space. Assume for any $\epsilon>0$, there exists $D(\epsilon)>0$ such that for all $d>D$, with probability at least $1-o(n)$ we have $|a_{n,d}-a_n|<\epsilon$. And for any $d\in\mathbb N$ we have $\lim_{n\to\infty} a_{n, d}=a_d$ almost surely, then the following convergence holds in probability:
		\[
			\lim_{n\to\infty} \lim_{d\to\infty} a_{n, d}=\lim_{d\to\infty} \lim_{n\to\infty} a_{n, d}.
		\]  
	\end{lemma}
	\begin{proof}
		We sometimes also write $a_d(n)$ to stress that we consider the sequence as a function of $n$. By assumption, let $n\to\infty$ we get the following statement holds with probability $1$:
		\[
			d, e>D\Longrightarrow |a_d-a_e|<\epsilon.
		\]
		This shows that $a_d:=\lim_{n\to\infty}a_{n, d}$ is a Cauchy sequence and have a finite limit $\lim_{d\to\infty}a_d=L$.
		
		Now define $a(n):=a_n=\lim_{d\to\infty}a_{n,d}$, for $d>D(\epsilon)$:
		\[
			\abs{a(n)-L}\leq \underbrace{|a(n)-a_d(n)|}_\text{A}+\underbrace{|a_d(n)-a_d|}_\text{B}+\underbrace{|a_d-L|}_\text{C}.
		\]
		By assumption, pick large enough $n$, we have $A<\epsilon$ with probability at least $1-o(n)$. By the Cauchy sequence argument above, we have $C<\epsilon$ with probability $1$. Finally since $a_d(n)\to a_d$ pointwise for every $d$, we can choose $n$ large enough such that $B<\epsilon$ with probability at least $1-o(n)$. Overall this gives
		\[
			P\p{|a(n)-L|>3\epsilon}<o(n),
		\]
		which concludes our proof
	\end{proof}

By standard high-dimensional probability \citep{vershynin2019high}, the following lemma holds:
\begin{lemma}\label{lma:highdimgaussian}
	Let $A\in\R^{n\times m}$ be a random matrix whose entries are sampled from i.i.d standard Gaussian distribution, then for $t\geq 0$, with probability at least $1-e^{-ct^2}$ for a constant $c>0$, there is:
	\[
		\|A\|_2\leq\sqrt{n}+\sqrt{m}+t
	\]
\end{lemma}

We are now ready to give the formal proof.
\exchangelimit*
\begin{proof}[Proof of \cref{thm:exchangelimit}]
	
	For any fixed $d$, we write $\Theta^{(d)} = \lim_{n\to\infty} \Theta_{n}^{(d)}$, notice this is just a finite-depth NTK (possibly with input injection). We condition on the event that $\lim_d \Theta_{n}^{(d)}$ exists. A sufficient condition for this event to hold with high probability is $\sigma_W^2< 1/8$. With such $\sigma_W^2$, by \cref{lma:highdimgaussian},  $\sigma\circ \sqrt{\sigma_W^2/n}W$ has a Lipschitz constant less than $1$ with high probability. Recall that $\sigma(x)=\sqrt{2}\max\{0, x\}$ is the normalized ReLU nonlinearity. Conditioned on such event, we have
	\begin{align*}
		&\frac{\partial f(x)}{\partial W^{(h)}}^T\frac{\partial f(x')}{\partial W^{(h)}}\\
		&=g^{(h-1)}(x)^T g^{(h-1)}(x')\cdot p^{(h)}(x)^Tp^{(h)}(x')\\
		&\leq \|g^{(h-1)}(x)\|\| g^{(h-1)}(x')\|\|p^{(h)}(x)\|\|p^{(h)}(x')\|\\
	\end{align*}
	WLOG let $g^{(0)}=x\in \S^{d-1}$, and $\|g^{(0)}\|\leq 1$ be our base case. Note that $U^{(h)}x$ is fixed for weight-tied network, let's denote it as $C$, and also overload the notation that $\|C\|=C$. By induction:
	\begin{align*}
		\norm{g^{(h)}}&=\norm{\sigma \p{f^{(h)}}} = \norm{\sigma \p{\sqrt{\frac{\sigma_W^2}{n}}W^{(h)} g^{(h-1)}+C  } } \\
		&\leq \norm {    \sqrt{\frac{2\sigma_W^2}{n}}W^{(h)} g^{(h-1)}+C       }\leq   \sqrt{\frac{2\sigma_W^2}{n}} \norm {   W^{(h)}}_{op} \norm{g^{(h-1)}}_2+\norm{C}
	\end{align*}
	By \cref{lma:highdimgaussian}, with probabiliy at least $1-e^{-\gO(t^2)}$, we have $\|W\|_{op}\leq 2\sqrt{n}+t$. This shows that for all $\epsilon>0$, let $\sigma_W<\frac{1}{2\sqrt 2+\epsilon}$, with probability at least $1-e^{-\gO(\epsilon^2n)}$, we have 
	\[
		\sqrt{\frac{2\sigma_W^2}{n}} \norm {   W^{(h)}}_{op} \triangleq r < 1.
	\]
	
	Consequently:
	\[
		\|g^{(h)}\|\leq r\|g^{(h-1)}\|+C\leq r^h\|g^{(0)}\|+ \sum_{l=1}^hCr^l, 
	\]
	which is geometric and converges absolutely as $h\to\infty$. Therefore, there exists a constant $Q>0$ s.t $\|g^{(h)}\|<Q$ for all $h\in\mathbb N$.
	
	By the same spirit, using induction, we have
	\[
		\|p^{(h)}\|\leq \frac{\sqrt{2\sigma_W^2}}{\sqrt n} \|W^{(h)}\|_{op} \|p^{(h+1)}\|\leq r\|p^{(h+1)}\|\leq r^{d-h}\|p^{(d+1)}\|=r^{d-h}.
	\]

	Combining the above two derivations, we have
	\begin{align*}
		&\sum_{h=1}^\infty\frac{\partial f(x)}{\partial W^{(h)}}^T\frac{\partial f(x')}{\partial W^{(h)}} \leq \sum_{h=1}^\infty\norm{\frac{\partial f(x)}{\partial W^{(h)}}}\sum_{h=1}^\infty\norm{\frac{\partial f(x')}{\partial W^{(h)}}}\\
		&\leq \p{\sum_{h=1}^\infty \|g^{(h-1)}(x)\|\|p^{(h)}(x)\|}\p{ \sum_{h=1}^\infty \|g^{(h-1)}(x')\|\|p^{(h)}(x')\|}<\infty.
	\end{align*}
	Similar convergence result can be derived for $\frac{df}{dU}$ as well. 
	
	Use the terminology introduced in \cref{lma:mo},  $\lim_{d\to\infty}\Theta_{n}^{(d)}=\lim_{d\to\infty}\Theta^{(d)}(n)=\sum_{h=1}^\infty \frac{\partial f(x)}{\partial \theta^{(h)}}^T \frac{\partial f(x')}{\partial \theta^{(h)}}$ converges uniformly in $n$ with high probability.

	For a fixed $d$, we know that $\lim_{n\to\infty}\Theta_{n}^{(d)} = \Theta^{(d)}$ by the tensor program \citep{yang2019tensor}. Therefore conditioned on the event that $\sigma\circ \sqrt{\sigma_W^2/n}W$ has a Lipschitz constant less than $1$, by \cref{lma:mo}, we can swap the limit and indeed $\lim_{d\to\infty}\lim_{n\to\infty} \Theta_{n}^{(d)}=\lim_{n\to\infty}\lim_{d\to\infty} \Theta_{n}^{(d)}$. This shows that the NTK-of-DEQ and the DEQ-of-NTK coincide.
\end{proof}

One should note that it merely requires $\sigma_W^2<1$ for the DEQ-of-NTK to converge as in \cref{thm:depntkconverge}, but our above proof requires $\sigma_W^2<1/8$ to make sure that the NTK-of-DEQ and DEQ-of-NTK are equivalent. Our current analysis relies heavily on a contraction argument. However, in the actual DEQ setting, it suffice to have $W$ being strongly monotone to guarantee convergence. That is, one only needs the largest eigenvalue of $W$ to be less than $1$. This corresponds to have $\sigma_W^2<1/2$ (again, this is because we use the normalized ReLU, so there is an extra factor of $\sqrt{2}$) by the semicircular law. We leave the gap to future works.

\section{Details of \Cref{sec:lineardeq}}

\linearexchange*
\begin{proof}[Proof of \Cref{thm:linearexchange}]
		Recall that we define $H:=\p{I-\sqrt{\frac{\sigma_W^2}{n}}W}^{-1}$. This inverse matrix exists with high probability if $\sigma_W^2<1/4$, due to a well-known random matrix theory result \cref{lma:highdimgaussian}. straightforward derivation gives:
	\begin{align*}
	\begin{split}
		&\lim_{d\to\infty}\ip{\frac{\partial f_n^{(d)}(x)}{\partial W}}{\frac{\partial f_n^{(d)}(y)}{\partial W}}\\
		&=\frac{\sigma_U^2\sigma_v^2}{n}\frac{\sigma_W^2}{n}\ip{Hv(HUx)^T}{Hv(HUx)^T}\\
		&=\underbrace{\frac{\sigma_W^2\sigma_U^2}{n}\ip{HUx}{HUx}\frac{\sigma_v^2}{n}\ip{Hv}{Hv}}_\text{A}\\
		&\xrightarrow{p}\underbrace{\sigma_U^2\sigma_W^2\sigma_v^2 x^Ty \p{\frac{1}{n}\tr\p{H^TH}}^2}_\text{B}\\
		&\xrightarrow{} \sigma_U^2\sigma_W^2\sigma_v^2 x^Ty\p{\int\frac{1}{\lambda}d\mu(\lambda)}^2.
	\end{split}
	\end{align*}
	The first convergence happens with high probability \citep{arora2019exact}. Note that $B=\E_{U, v}[A]$. One needs to apply the Gaussian chaos of order 2 lemma \citep{boucheron2013concentration} to show the concentration. This was done rigorously down in \cite{arora2019exact} Claim E.2. Their proof works for our case as well since we have $\|H^TH\|_2$ bounded independently of $n$ and $d$ with high probability.

	The second convergence holds for almost every realization of a sequence of $W$. Recall that $\mu_n$ is the empirical distribution of the eigenvalue of the matrix $\p{I-\sqrt{\frac{\sigma_W^2}{n}}W}^{T}\p{I-\sqrt{\frac{\sigma_W^2}{n}}W}$. More precisely, $\mu_n=\frac{1}{n} \sum_{i=1}^{n} \delta_{\lambda_{i}}$, $\delta_{\lambda_i}$ is the delta measure at the $i$th eigenvlue value $\lambda_i$. We can rewrite 
	\[
		\frac{1}{n}\tr\p{H^TH}=\int \frac{1}{\lambda}d\mu_n(\lambda).
	\]
	 We will show that $\mu_n\to\mu$ weakly a.s \footnote{Note here $\mu_n$ is a random measure}. Then by Portmanteau lemma, we have $\int fd\mu_n\to \int fd\mu$ for every bounded Lipschitz function. Here we have $f=1/\lambda$ defined when $\lambda$ has non-zero support in $\mu(\lambda)$. Since by \cref{lma:highdimgaussian}, our assumption $\sigma_W^2<1/8$ guarantees $\norm{\sqrt{\frac{\sigma_W^2}{n}}W}<1$ w.h.p, the support of $\mu(\lambda)$ is bounded away from $0$, and $f$ is indeed Lipschitz and bounded on its domain.

	Next, we show that $\int\frac{1}{\lambda}d\mu(\lambda)=\frac{1}{1-\sigma_W^2}$. From \citet{capitaine2016spectrum}, we learn that the Stieltjes transform $g$ of $\mu$ is a root to the following cubic equation:
	\begin{align}\label{eq:cubicroot}
		\text{For }z\in\mathbb C^+: g_\mu(z)^{-1} = \left(1- \sigma_W^{2} g_{\mu}(z)\right) z-\frac{1}{1- \sigma_W^{2} g_{\mu}(z)}.
	\end{align}
	
	Deducing the probability density from $g$ by using the inverse formula of Stieltjes transformation, we have
	
	\resizebox{\linewidth}{!}{
	  \begin{minipage}{\linewidth}
		\begin{align*}
			p(b)&=\lim_{b\to 0^+}\frac{1}{\pi}\operatorname{Im}(g(a+bi)\\
			&=\frac{1}{\pi}\Bigg( \frac{\sqrt{3} \left(3  \sigma_W^6 b- \sigma_W^4 b^2-3  \sigma_W^4 b\right)}{3\ 2^{2/3}  \sigma_W^4 b \p{9  \sigma_W^8 b^2-2  \sigma_W^6 b^3+18  \sigma_W^6 b^2+\sqrt{\left(9  \sigma_W^8 b^2-2  \sigma_W^6 b^3+18  \sigma_W^6 b^2\right)^2+4 \left(3  \sigma_W^6 b- \sigma_W^4 b^2-3  \sigma_W^4 b\right)^3}}^{1/3}}+\\
			&\qquad\frac{\sqrt{3} \p{9  \sigma_W^8 b^2-2  \sigma_W^6 b^3+18  \sigma_W^6 b^2+\sqrt{\left(9  \sigma_W^8 b^2-2  \sigma_W^6 b^3+18  \sigma_W^6 b^2\right)^2+4 \left(3  \sigma_W^6 b- \sigma_W^4 b^2-3  \sigma_W^4 b\right)^3}}^{1/3}}{6 \sqrt[3]{2}  \sigma_W^4 b} \Bigg)
		\end{align*}
	  \end{minipage}
	}

	Finally we can compute $\int_l^u \frac{1}{\lambda}p(\lambda)d\lambda$. Notice to let $p(\cdot)$ be well defined, we need $9  \sigma_W^8 b^2-2  \sigma_W^6 b^3+18  \sigma_W^6 b^2\geq 0$, which amounts to $l=\frac{1}{8} \left(-\sigma_W^4+20 \sigma_W^2-\sqrt{\sigma_W^8+24 \sigma_W^6+192 \sigma_W^4+512 a^2}+8\right)$ and $u=\frac{1}{8} \left(-\sigma_W^4+20 \sigma_W^2+\sqrt{\sigma_W^8+24 \sigma_W^6+192 \sigma_W^4+512 a^2}+8\right)$. This now involves a one-dimensional integral, which an be solved numerically for all values of $\sigma_W$, and shown be be arbitrarily close the desired quantity $1/(1-\sigma_W^2)$.

	Similarly, we can compute that 
	\[
		\lim_{d\to\infty}\ip{\frac{\partial f_n^{(d)}(x)}{\partial U}}{\frac{\partial f_n^{(d)}(y)}{\partial U}}\xrightarrow{p}  \frac{\sigma_v^2\sigma_U^2x^Ty}{1-\sigma_W^2}
	\] and 
	\[
		\lim_{d\to\infty}\ip{\frac{\partial f_n^{(d)}(x)}{\partial v}}{\frac{\partial f_n^{(d)}(y)}{\partial v}}\xrightarrow{p}  \frac{\sigma_v^2\sigma_U^2x^Ty}{1-\sigma_W^2}.
	\]
	Summing the three relevant terms and use the fact that $\sigma_U^2+\sigma_W^2=1$, we get the claimed result.
\end{proof}

\section{DEQ with Convolution Layers}\label{sec:cdeqntk}
In this section we show how to derive the NTKs for convolution DEQs (CDEQ). Although in this paper only the CDEQ with vanilla convolution structure is considered, we remark that our derivation is general enough for other CDEQ structures as well, for instance, CDEQ with global pooling layer. The details of this section can be found in the appendix.

Unlike the fully connection network with input injection, whose intermediate NTK representation is a real number. For convolutional neural networks (CNN), the intermediate NTK representation is a four-way tensor. In the following, we will present the notations, CNN with input injection (CNN-IJ) formulation, the CDEQ-NTK initialization, and our main theorem.

\paragraph{Notation.} We adopt the notations from \citet{arora2019exact}. Let $x, y\in\R^{P\times Q}$ be  a pair of inputs, let $q\in\mathbb Z_+$ be the filter size (WLOG assume it is odd as well). By convention, we always pad the representation (both the input layer and hidden layer) with $0$'s. Denote the convolution operation for $i \in[P], j \in[Q]$:
$
	 [w *x]_{i j}=\sum_{a=-\frac{q-1}{2}}^{\frac{q-1}{2}} \sum_{b=-\frac{q-1}{2}}^{\frac{q-1}{2}}[{w}]_{a+\frac{q+1}{2}, b+\frac{q+1}{2}}[{x}]_{a+i, b+j} .
$

Denote

\resizebox{\linewidth}{!}{
  \begin{minipage}{\linewidth}
  \begin{align*}
	\mathcal{D}_{i j, i^{\prime} j^{\prime}} =\Big\{&\left(i+a, j+b, i^{\prime}+a^{\prime}, j^{\prime}+b^{\prime}\right) \in[P] \times[Q] \times[P] \times[Q] : -(q-1) / 2 \leq a, b, a^{\prime}, b^{\prime} \leq(q-1) / 2\Big\}.
\end{align*}
  \end{minipage}
}

%
Intuitively, $\mathcal{D}_{i j, i^{\prime} j^{\prime}}$ is a $q\times q\times q\times q$ set of indices centered at $(ij,i' j')$. For any tensor $T\in\R^{P\times Q\times P\times Q}$, let $[T]_{\cD_{ij, i'j'}}$ be the natural sub-tensor and let $\Tr(T)=\sum_{i, j}T_{ij,ij}$. 

\paragraph{Formulation of CNN-IJ.} Define the CNN-IJ as follows:
\begin{itemize}[leftmargin=*]
	\item Let the input $x^{(0)}=x\in\R^{P\times Q\times C_0}$, where $C_0$ is the number of input channels, and $C_h$ is the number of channels in layer $h$. Assume WLOG that $C_h=C$ for all $h\in[d]$
	\item For $h=1,\ldots, d$, let the inner representation
		\begin{align}
			&\tilde x_{(\beta)}^{(h)}  = \sum_{\alpha=1}^{C_{h-1}} \sqrt{\frac{\sigma_W^2}{C_h}}W^{(h)}_{(\alpha),(\beta)}*x_{(\alpha)}^{(h-1)}+\sum_{\alpha=1}^{C_0} \sqrt{\frac{\sigma_U^2}{C_h}} U^{(h)}_{(\alpha),(\beta)}*x_{(\alpha)}^{(0)} \\ 
			& \left[x_{(\beta)}^{(h)}\right]_{ij}= \frac{1}{[S]_{ij}}\left[\sigma\left(\tilde x_{(\beta)}^{(h)}\right)\right]_{ij},\ \text{for } i\in[P], j\in[Q] \label{eq:cntkq}
		\end{align}
		where $W^{(h)}_{(\alpha),(\beta)}\in\R^{q\times q}$ represent the convolution operator from the $\alpha^{th}$ channel in layer $h-1$ to the $\beta^{th}$ channel in layer $h$. Similarly, $U^{(h)}_{(\alpha),(\beta)}\in\R^{q\times q}$ injects the input in each convolution window. $S\in\R^{P\times Q}$ is a normalization matrix. Let $W, U, S,\sigma_U^2,\sigma_W^2$ be chosen by the CDEQ-NTK initialization described later. 
	\item 
		The final output is defined to be
		$
			f_\theta(x) = \sum_{\alpha=1}^{C_d}\ip{ W^{(d+1)}_{(\alpha)}}{x_{(\alpha)}^{(d)}} ,
		$
		where $ W^{(d+1)}_{(\alpha)}\in \R^{P\times Q}$ is sampled from standard Gaussian distribution.
\end{itemize}

\paragraph{CDEQ-NTK initialization.} Let $1_q\in\R^{q\times q}, X\in\R^{P\times Q}$ be two all-one matrices. Let $\tilde X\in\R^{(P+2)\times (Q+2)}$ be the output of zero-padding $X$. We index the rows of $\tilde X$ by $\{0, 1, \ldots, P+1\}$ and columns by $\{0,1,\ldots, Q+1\}$. For position $i\in [P], j\in[Q]$, let $\p{[S]_{ij}}^2=[1_q * \tilde X]_{ij}$ in \cref{eq:cntkq}. Let every entry of every $W, U$ be sampled from $\cN(0, 1)$ and $\sigma_W^2+\sigma_U^2=1$.

Using the above-defined notations, we now state the CDEQ-NTK.

\begin{theorem}\label{thm:cdepntk_main}
	Let $x,y\in\R^{P\times Q\times C_0}$ be s.t $\|x_{ij}\|_2=\|y_{ij}\|_2=1$ for $i\in[P], j\in[Q]$. Define the following expressions recursively (some $x,y$ are omitted in the notations), for $(i,j,i',j')\in [P]\times[Q]\times[P]\times[Q]$, $h\in[d]$
	\begin{align}
		&{K}_{ij,i'j'}^{(0)}\left({x}, y\right)=\left[\sum_{\alpha\in[C_0]} x_{(\alpha)}\otimes y_{(\alpha)}\right ]_{ij,i'j'} \\
		\begin{split}
			&\left[{\Sigma}^{(0)}\left({x}, y\right)\right]_{i j, i^{\prime} j^{\prime}}=\frac{1}{[S]_{ij}[S]_{i'j'}}\sum_{\alpha=1}^{C_0} \operatorname{Tr}\left(\left[{K}_{(\alpha)}^{(0)}\left({x}, y\right)\right]_{\mathcal{D}_{i j, i^{\prime} j^{\prime}}}\right)\\
		\end{split}
	\end{align}
	
	\begin{align}
		\begin{split}
			&\R^{2\times 2}\ni{\Lambda}_{i j, i^{\prime} j^{\prime}}^{(h)}\left({x}, y\right)=\left(\begin{array}{cc}
			{\left[{\Sigma}^{(h-1)}({x}, {x})\right]_{i j, i j}} & {\left[{\Sigma}^{(h-1)}\left({x}, y\right)\right]_{i j, i^{\prime} j^{\prime}}} \label{eq:cdeqntklambda}\\
			{\left[{\Sigma}^{(h-1)}\left(y, {x}\right)\right]_{i^{\prime} j^{\prime}, i j}} & {\left[{\Sigma}^{(h-1)}\left(y, y\right)\right]_{i^{\prime} j^{\prime}, i^{\prime} j^{\prime}}}
		\end{array}\right) \\
		\end{split}
	\end{align}
	
	\begin{align}
		\begin{split}\label{eq:convK}
			&\left[ K^{(h)}(x, y)\right]_{ij,i'j'}=\frac{\sigma_W^2 }{[S]_{ij}\cdot [S]_{i'j'}}\mathop{\E}_{\substack{(u, v)\\ \sim\cN(0, \Lambda^{(h)}_{ij,i'j'})}}[\sigma(u)\sigma(v)] +  \frac{\sigma_U^2}{[S]_{ij}\cdot [S]_{i'j'}}[K^{(0)}]_{ij, i'j'} 
		\end{split}
		\\
		&\left[\dot K^{(h)}(x, y)\right]_{ij,i'j'}=\frac{\sigma_W^2 }{[S]_{ij}\cdot [S]_{i'j'}}\mathop{\E}_{\substack{(u, v)\\ \sim\cN(0, \Lambda^{(h)}_{ij,i'j'})}}[\dot\sigma(u)\dot\sigma(v)]\\
		& \left[\Sigma^{(h)}(x,y)\right]_{ij,i'j'}=\Tr\p{ \left[ K^{(h)}(x, y) \right]_{\cD_{ij,i'j'}}  }\label{eq:convSigma}
	\end{align}
		
		Define the linear operator $\mathcal{L}: \mathbb{R}^{P \times Q \times P \times Q} \rightarrow \mathbb{R}^{P \times Q \times P \times Q}$ via 
		$
			[\mathcal{L}(M)]_{ij,i'j'}= \operatorname{Tr}\left([M]_{\mathcal{D}_{ij,i'j'}}\right).
		$
		
		Then the CDEQ-NTK can be found solving the following linear system:
		\begin{align}\label{eq:equilibriumcntk}
		\begin{split}
			\Theta^*(x,y) &= \dot K^*(x,y)\odot \cL\p{\Theta^*(x,y) }+K^*(x,y),\\
		\end{split}	
		\end{align}
		where $K^*(x,y) = \lim_{d\to\infty }K^{(L)}(x,y), \dot K^*(x,y) = \lim_{d\to\infty }\dot K^{(d)}(x,y)$. The limit exists if $\sigma_W^2<1$. The actual NTK entry is calculated by $\Tr(\Theta^*(x,y))$.
\end{theorem}

\Cref{thm:cdepntk_main} highlights that the convergence of CDEQ-NTK depends solely on the CDEQ-NTK initialization. The crucial factor here is the normalization tensor $S$, which guarantees the variance of each term is always $1$ across the propogation. This idea mimics that of the DEQ-NTK initialization. Our theorem shows that CDEQ-NTK can also be computed by solving fixed point equations.

%

We first explain the choice of $S$ in the CDEQ-NTK initialization. In the original CNTK paper \citep{arora2019exact}, the normalization is simply $1/q^2$. However, due to the zero-padding, $1/q^2$ does not normalize all $\left[\Sigma^{(h)}(x,x)\right]_{ij,i'j'}$ as expected: only the variances that are away from the corners are normalized to $1$, but the ones near the corner are not. $[S]_{ij}$ is simply the number of non-zero entries in $\left[\tilde X\right]_{\cD_{ij,ij}}$. 

	Now we give the proof to \Cref{thm:cdepntk_main}.

\begin{proof}[Proof of \cref{thm:cdepntk_main}]
	Similar to the proof of \cref{thm:deqntk}, we can split the CDEQ-NTK in two terms:
	\begin{align*}
			&\Theta^{(L)}(x,y) = \E_{\theta}\left[ \ip{\frac{\partial f(\theta ,x)}{\partial \theta} }{\frac{\partial f(\theta ,y)}{\partial \theta} }\right]\\
			=&\underbrace{\E_{\theta}\left[ \ip{\frac{\partial f(\theta ,x)}{\partial W} }{\frac{\partial f(\theta ,y)}{\partial W} }\right]}_\text{$\circled{1}$}+\underbrace{\E_{\theta}\left[ \ip{\frac{\partial f(\theta ,x)}{\partial U} }{\frac{\partial f(\theta ,y)}{\partial U} }\right]}_\text{$\circled{2}$}.
	\end{align*}
	Omit the input symbols $x,y$, let
	\[
		\left[ \wh K^{(h)} \right]_{ij,i'j'}   = \frac{\sigma_W^2 }{[S]_{ij}\cdot [S]_{i'j'}}\mathop{\E}_{(u, v)\sim\cN(0, \Lambda^{(h)}_{ij,i'j'})}[\sigma(u)\sigma(v)].
	\]
	
	As shown in \citet{arora2019exact}, we have
	\[
		\left\langle\frac{\partial f_\theta({x})}{\partial {W}^{(h)}}, \frac{\partial f_\theta\left(, y\right)}{\partial {W}^{(h)}}\right\rangle \to \operatorname{Tr}\left(\dot{{K}}^{(d)} \odot \mathcal{L}\left(\dot{{K}}^{(d-1)} \odot \mathcal{L}\left(\cdots \dot{{K}}^{(h)} \odot \mathcal{L}\left( \wh K^{h-1}  \right) \cdots\right)\right)\right)
	\]
	Write $f=f_\theta(x)$ and $\tilde f = f_\theta(y)$. Following the same step, by chain rule, we have
	\begin{align*}
		&\ip{      \frac{  \partial f }{ \partial U^{(h)} }       }{      \frac{  \partial \tilde f }{ \partial U^{(h)} }        }\to \operatorname{Tr}\left( \dot K^{(d)} \odot \mathcal{L}\left(\dot{{K}}^{(d-1)} \odot \mathcal{L}\left(\cdots \dot{{K}}^{(h)} \odot \mathcal{L}\left({K}^{(0)}\right) \cdots\right)\right)\right)
	\end{align*}
	Rewrite the above two equations in recursive form, we can calculate the $L$-depth iteration of CDEQ-NTK by:
		\begin{itemize}
			\item For the first layer $\Theta^{(0)}(x,y) = \Sigma^{(0)}(x,y)$.
			\item 
				For $h=1,\ldots, d-1$, let
				\begin{align}\label{eq:cdeqntk}
					\left[{\Theta}^{(h)}\left({x}, y\right)\right]_{i j, i^{\prime} j^{\prime}}=\operatorname{Tr}\left(\left[\dot{{K}}^{(h)}\left({x}, y\right) \odot {\Theta}^{(h-1)}\left({x}, y\right)+{K}^{(h)}\left({x}, y\right)\right]_{\cD_{i j, i^{\prime} j^{\prime}}}\right)
				\end{align}
			\item 
				For $h=d$, let
				\begin{align}\label{eq:cntklastlayer}
					{\Theta}^{(L)}\left({x}, y\right)=\dot{{K}}^{(d)}\left({x}, y\right) \odot {\Theta}^{(d-1)}\left({x}, y\right)+{K}^{(h)}\left({x}, y\right) 
				\end{align}
			\item The final kernel value is $\Tr(\Theta^{(d)}(x,y))$.
		\end{itemize}
		
		 Using \cref{eq:cdeqntk} and \cref{eq:cntklastlayer}, we can find the following recursive relation:
		 \begin{align}
		 	\Theta^{(d+1)}(x,y) = \dot K^{(d+1)}(x,y)\odot \cL\p{\Theta^{(d)}(x,y) }+K^{(h+1)}(x,y)\label{cref:cdeqntklastlinear}
		 \end{align}
		 
		 The rest of the proof is stated in the main text. For readers' convenience we include them here again.
		 
		At this point, we need to show that $K^*(x,y)\triangleq\lim_{d\to\infty} K^{(d)}(x,y)$ and $\dot K^*(x,y)\triangleq\lim_{d\to\infty} \dot K^{(d)}(x,y)$ exist. Let us first agree that for all $h\in[d]$, $(ij,i'j')\in [P]\times [Q]\times  [P]\times [Q]$, the diagonal entries of $\Lambda^{(h)}_{ij,i'j'}$ are all ones. Indeed, these diagonal entries are $1$'s at $h=0$ by initialization. Note that iterating \crefrange{eq:cdeqntklambda}{eq:convSigma} to solve for $[\Sigma^{(h)}(x,y)]_{ij,i'j'}$ is equivalent to iterating $f:\R^{P\times Q\times P\times Q}\to \R^{P\times Q\times P\times Q}$:
		\begin{align}
			P^{(h+1)}=f(P^{(h)})\triangleq\cL\p{\frac{1}{[S]_{ij}[S]_{i'j'}}R_\sigma(P^{(h)})}, P^{(0)}=K^{(0)}
		\end{align}
		where 
		\begin{align}
			R_\sigma(P^{(h)}_{ij,i'j'})\triangleq \sigma_W^2 \left(\frac{\sqrt{1-\p{P^{(h)}_{ij,i'j'}}^{2}}+\left(\pi-\cos ^{-1}\p{P^{(h)}_{ij,i'j'}}\right) P^{(h)}_{ij,i'j'}}{\pi}\right)+ \sigma_U^2 K^{(0)}_{ij,i'j'}
		\end{align}
		is applied to $P^{(h)}$ entrywise.
		
		Due to CDEQ-NTK initialization, if $P^{(0)}_{ij,ij}=1$ for $i\in[P],j\in[Q]$, then $P^{(h)}_{ij,ij}=1$ for all iterations $h$. This is true by the definition of $S$. 
		
		Now if we can show $f$ is a contraction, then $\Sigma^*(x,y)\triangleq \lim_{h\to\infty} \Sigma^{(h)}(x,y)$ exists, hence $K^*$ and $\dot K^*$ also exist. We should keep the readers aware that $f:\R^{P\times Q\times P \times Q}\to \R^{P\times Q\times P \times Q}$, so we should be careful with the metric spaces. We want every entry of $\Sigma^{(h)}(x,y)$ to converge, since this tensor has finitely many entries, this is equivalent to say its $\ell^\infty$ norm (imagine flattenning this tensor into a vector) converges. So we can equip the domain an co-domain of $f$ with $\ell^\infty$ norm (though these are finite-dimensional spaces so we can really equip them with any norm, but picking $\ell^\infty$ norm makes the proof easy).  
		
		Now we have $f=\cL\circ \frac{1}{[S]_{ij}[S]_{i'j'}} R_\sigma :\ell^\infty\to\ell^\infty$. If we flatten the four-way tensor $P^{(h)}$ into a vector, then $\cL$ can be represented by a $(P\times Q\times P \times Q)\times (P\times Q\times P \times Q)$ dimensional matrix, whose $(kl,k'l')$-th entry in the $(ij,i'j')$-th row is $1$ if $(kl,k'l')\in\cD_{ij,i'j'}$, and $0$ otherwise. In other words, the $\ell^1$ norm of the $(ij,i'j')$-th row represents the number of non-zero entries in $\cD_{ij,i'j'}$, but by the CDEQ-NTK initialization, the row $\ell^1$ norm divided by $[S]_{ij}\cdot [S]_{i'j'}$ is at most $1$! Using the fact that $\|\cL\|_{\ell^\infty\to\ell^\infty}$ is the maximum $\ell^1$ norm of the row, and the fact $R_\sigma$ is a contraction (proven in \cref{thm:depntkconverge}), we conclude that $f$ is indeed a contraction.
		
		With the same spirit, we can also show that \cref{eq:cntklastlayer} is a contraction if $\sigma_W^2<1$, hence \cref{eq:equilibriumcntk} is indeed the unique fixed point. This finishes the proof.
	\end{proof}

\subsection{Computation of CDEQ-NTK}
%
	
	One may wish to directly compute a fixed point (or more precisely, a  fixed tensor) of $\Theta^{(d)}\in\R^{P\times Q\times P\times Q}$ like \cref{eq:dualactivationsigmawithinjection}. However, due to the linear operator $\cL$ (which is just the ensemble of the trace operator in \cref{eq:convSigma}), the entries depend on each other. Hence the system involves a $(P\times Q\times P\times Q)\times (P\times Q\times P\times Q)$-dimensional matrix that represents $\cL$. Even if we exploit the fact that only entries on the same ``diagonal'' depend on each other, $\cL$ is at least $P\times Q\times P\times Q$, which is $32^4$ for CIFAR-10 data.
	
	Moreover, this system is nonlinear. Therefore we cannot compute the fixed point $\Sigma^*$ by root-finding efficiently. Instead, we approximate it using finite depth iterations, and we observe that in experiments they typically converge to $10^{-6}$ accuracy in $\ell^\infty$ within $15$ iterations.

\begin{table*}[t]
\vskip 0.1in
\caption{Performance of CDEQ-NTK on CIFAR-10 dataset}\label{table:cifarcntk} 
\vskip 0.1in
\centering	
\begin{tabular}{ ccc } 
\midrule
\textbf{Method} & \textbf{Parameters} & \textbf{Acc.} \\
\midrule
CDEQ-NTK with $2000$ training data & $\sigma_W^2=0.65, \sigma_U^2=0.35$ & $37.49\% $\\
CNTK with $2000$ training data & Depth = 6 & $43.43\% $\\
CNTK with $2000$ training data & Depth = 21 & $42.53\% $\\
\hline
\end{tabular}
\vskip 0.1in
\end{table*}	

We test CDEQ-NTK accuracy on CIFAR-10 dataset with just 2000 training data. The result is shown in \Cref{table:cifarcntk}.

\begin{table*}[t]
\vskip 0.1in
\caption{Performance of DEQ-NTK on CIFAR-10 dataset, see \citet{lee2020finite} for NTK with ZCA regularization..}\label{table:cifar} 
\vskip 0.1in
\centering	
\begin{tabular}{ ccc } 
\midrule
\textbf{Method} & \textbf{Parameters} & \textbf{Acc.} \\
\midrule
DEQ-NTK &  $\sigma_W^2=0.25, \sigma_U^2=0.25, \sigma_b^2=0.5$ & $59.08\%$ \\ \addlinespace[1pt]
DEQ-NTK &  $\sigma_W^2=0.6, \sigma_U^2=0.4, \sigma_b^2=0$ & \textbf{59.77\%} \\ \addlinespace[1pt]
DEQ-NTK &  $\sigma_W^2=0.8, \sigma_U^2=0.2, \sigma_b^2=0$ & $59.43\%$ \\ \addlinespace[1pt]
NTK with ZCA regularization& $\sigma_W^2=2, \sigma_b^2=0.01$ & $59.7\%$ \\ 

\hline
\end{tabular}
\vskip 0.1in
\end{table*}

\begin{table}[h!]
\vskip 0.1in
	\caption{Performance of DEQ-NTK on MNIST dataset, compared to neural ODE \citep{chen2018neural} and monotone operator DEQ, see these results from \cite{winston2020monotone}. }\label{table:mnist}
	\vskip 0.1in
	\centering	
	\begin{tabular}{ ccc } 
	\toprule
	\multicolumn{3}{c}{\textbf{MNIST}}\\ \midrule
	\textbf{Method} & \textbf{Model size} & \textbf{Acc.} \\
	\hline
	DEQ-NTK &  & \textbf{98.6\%} \\ \addlinespace[1pt]
	Neural ODE& 84K & 98.2\% \\ \addlinespace[1pt]
	MON DEQ& 84K & 98.2\% \\ 
	\bottomrule
	\end{tabular}
\vskip 0.1in
\end{table}

\end{document}